\documentclass[10pt,twocolumn,letterpaper]{article}

\usepackage{cvpr}
\usepackage{times}
\usepackage{epsfig}
\usepackage{graphicx}
\usepackage{amsmath}
\usepackage{amssymb}

\usepackage{amsthm}
\usepackage{subfigure}
\usepackage{algorithm}
\usepackage{algorithmic}
\usepackage{epstopdf}
\usepackage{xr}
\usepackage[page]{appendix}
\graphicspath{ {./figs/} }

\usepackage{tikz,tikz-3dplot}
\usetikzlibrary{arrows}
\usetikzlibrary{calc,fit}
\usepackage{color}

\usepackage[pagebackref=true,breaklinks=true,letterpaper=true,colorlinks,bookmarks=false]{hyperref}

\newcommand\modify[1]{\textcolor{red}{#1}} 
\renewcommand\modify[1]{#1}

\def\ie{i.e.}
\def\eg{e.g.}

\def\st{~~\textrm{s.t.}~~}
\def\and{\textrm{and}}
\def\conv{\textrm{conv}}

\def\b{\textbf{b}}
\def\0{\textbf{0}}
\def\1{\textbf{1}}
\def\a{\boldsymbol{a}}
\def \bfdelta{\boldsymbol{\delta}}
\def\c{\boldsymbol{c}}
\def\e{\boldsymbol{e}}

\def\v{\boldsymbol{v}}

\def\w{\boldsymbol{w}}
\def\x{\boldsymbol{x}}

\def\A{\mathcal{A}}

\def\P{\mathcal{P}}

\def\T{\mathcal{T}}
\def\S{\mathcal{S}}

\def\transpose{\top}

\newcommand{\RR}{I\!\!R} 

\newcommand{\myparagraph}[1]{\smallskip\noindent\textbf{#1.}}

\DeclareMathOperator*{\argmin}{arg\,min}

\newtheorem{theorem}{Theorem}[section]
\newtheorem{problem}{Problem}[section]
\newtheorem{lemma}{Lemma}[section]

\newtheorem{proposition}{Proposition}[section]
\newtheorem{definition}{Definition}[section]
\newtheorem{remark}{Remark}[section]

\newtheorem*{lemma*}{Lemma}
\newtheorem*{theorem*}{Theorem}

\cvprfinalcopy 

\def\cvprPaperID{1235} 

\ifcvprfinal\fi
\begin{document}

\title{Elastic net subspace clustering: theoretical analysis and \\an oracle based active set algorithm}
\title{Efficient elastic net optimization applied to subspace clustering}

\title{Efficient Elastic Net Optimization and its Application for Scalable Subspace Clustering}
\title{Efficient Elastic Net Optimization and its Application to Scalable Subspace Clustering}
\title{Efficient Elastic Net Optimization and its Application to Fast Subspace Clustering}
\title{Efficient Active Set Algorithm for Elastic Net Subspace Clustering}
\title{Efficient Active Set Algorithm for Fast Elastic Net Subspace Clustering}
\title{Efficient Active Set Algorithm for Scalable Elastic Net Subspace Clustering}
\title{Efficient Active Set Algorithm for Provable and Scalable Elastic Net Subspace Clustering}
\title{Provable and Scalable Elastic Net Subspace Clustering}
\title{Scalable and Provable Elastic Net Subspace Clustering}

\title{Elastic Net Subspace Clustering: Theory and Scalable Oracle-Guided Algorithm}
\title{Elastic Net Subspace Clustering: \\Theoretical Analysis and an Oracle Based Active Set Algorithm}
\title{Elastic Net Subspace Clustering: Theory and Scalable Active Set Algorithm}
\title{Scalable and Provable Active Set Algorithm for Elastic Net Subspace Clustering}
\title{Scalable Elastic Net Subspace Clustering: An Oracle Based Active Set Algorithm}
\title{Oracle Based Active Set Algorithm for Scalable Elastic Net Subspace Clustering}

\author{Chong You$^\dag$ \quad Chun-Guang Li$^*$\quad Daniel P. Robinson$^\ddagger$\quad Ren\'e Vidal$^\dag$\\
$^\dag$Center for Imaging Science, Johns Hopkins University\\
$^*$SICE, Beijing University of Posts and Telecommunications\\
$^\ddagger$Applied Mathematics and Statistics, Johns Hopkins University\\
}

\maketitle
\thispagestyle{empty}

\begin{abstract}

State-of-the-art subspace clustering methods are based on expressing each data point as a linear combination of other data points while regularizing the matrix of coefficients with $\ell_1$, $\ell_2$ or nuclear norms. $\ell_1$ regularization is guaranteed to give a subspace-preserving affinity (i.e., there are no connections between points from different subspaces) under broad theoretical conditions, but the clusters may not be connected. $\ell_2$ and nuclear norm regularization often improve connectivity, but give a subspace-preserving affinity only for independent subspaces. Mixed $\ell_1$, $\ell_2$ and nuclear norm regularizations offer a balance between the subspace-preserving and connectedness properties, but this comes at the cost of increased computational complexity. This paper studies the geometry of the elastic net regularizer (a mixture of the $\ell_1$ and $\ell_2$ norms) and uses it to derive a provably correct and scalable active set method for finding the optimal coefficients. Our geometric analysis also provides a theoretical justification and a geometric interpretation for the balance between the connectedness (due to $\ell_2$ regularization) and subspace-preserving (due to $\ell_1$ regularization) properties for elastic net subspace clustering. Our experiments show that the proposed active set method not only achieves state-of-the-art clustering performance, but also efficiently handles large-scale datasets.

\vspace{-1em}
\end{abstract}

\section{Introduction}

In many computer vision applications, including image representation and compression \cite{Hong:TIP06}, motion segmentation \cite{Costeira:IJCV98,Vidal:IJCV08,Rao:PAMI10}, temporal video segmentation \cite{Vidal:PAMI05}, and face clustering \cite{Ho:CVPR03}, high-dimensional datasets can be well approximated by a union of low-dimensional subspaces. In this case, the problem of clustering a high-dimensional dataset into multiple classes or categories reduces to the problem of assigning each data point to its own subspace and recovering the underlying low-dimensional structure of the data, a problem known in the literature as \textit{subspace clustering} \cite{Vidal:SPM11-SC}.


\myparagraph{Prior Work}
Over the past decade, the subspace clustering problem has received a lot of attention in the literature
and many methods have been developed. Among them, spectral clustering
based methods have become extremely popular \cite{Chen:IJCV09,
  Zhang:IJCV12, Elhamifar:CVPR09, Elhamifar:TPAMI13, Liu:ICML10,
  Liu:TPAMI13, Favaro:CVPR11, Vidal:PRL14, Lu:ECCV12,Dyer:JMLR13,Heckel:arxiv13,Park:NIPS14,Li:CVPR15,Tsakiris:FSASCICCV15}
(see \cite{Vidal:SPM11-SC} for details). These methods usually divide
the problem into two steps: a) learning  an affinity matrix that characterizes whether two points are likely to lie in the same subspace, and b) applying spectral clustering to this affinity. Arguably, the first step is the most important, as the success of spectral clustering depends on having an appropriate affinity matrix.

State-of-the-art methods for constructing the affinity matrix are based on the \emph{self-expressiveness model} \cite{Elhamifar:CVPR09}.
Under this model, each data point $\x_j$ is expressed as a linear combination of all other data points, \ie, $\x_j = \sum_{i \ne j} \x_i c_{ij} + \e_j$, where the coefficient $c_{ij}$ is used to define an affinity between points $i$ and $j$, and the vector $\e_j$ captures deviations from the self-expressive model. The coefficients are typically found by solving an optimization problem of the form
\vspace{-0.5em}
\begin{equation}
\min _{\c_j, \e_j} r(\c_j) + \gamma \cdot h(\e_j) \st \x_j = X \c_j + \e_j, c_{jj} = 0,
\label{eq:self-expression}
\end{equation}
%
where $X = [\x_1, \cdots, \x_N]$ is the data matrix, $\c_j = [c_{1j}, \dots , c_{Nj}]^\transpose$ is the vector of coefficients, $r(\cdot)$ is a properly chosen regularizer on the coefficients, $h(\cdot)$ is a properly chosen regularizer on the noise or corruption, and $\gamma > 0$ is a parameter that balances these two regularizers.

The main difference among state-of-the-art methods lies in the choice of the regularizer $r(\cdot)$. The sparse subspace clustering (SSC) method \cite{Elhamifar:CVPR09} searches for a sparse representation using $r(\cdot) = \|\cdot\|_1$. While under broad theoretical conditions (see
\cite{Elhamifar:TPAMI13,Soltanolkotabi:AS13,You:ICML15}) the representation produced by SSC is guaranteed to be \emph{subspace preserving} (\ie,
$c_{ij} \neq 0$ only if $\x_i$ and $\x_j$ are in the same subspace), the affinity
matrix may lack \emph{connectedness} \cite{Nasihatkon:CVPR11} (\ie, data points from the same subspace may not form a connected component of the affinity graph due to the sparseness of the connections, which may cause over-segmentation). Other recently proposed sparsity based methods, such as orthogonal matching pursuit (OMP) \cite{Dyer:JMLR13,You:CVPR16-SSCOMP} and nearest subspace neighbor (NSN) \cite{Park:NIPS14}, also suffer from the same connectivity issue.

As an alternative, the least squares regression (LSR) method  \cite{Lu:ECCV12} uses the regularizer $r(\cdot) = \frac{1}{2}\| \cdot
\|_2^2$. One benefit of LSR is that the representation matrix is generally dense,
which alleviates the connectivity issue of sparsity based
methods. However, the representation is known to be subspace preserving only when the subspaces are independent,\footnote{Subspaces $\{\S_\kappa\}$
  are independent if $\dim (\sum_\kappa \S_\kappa) = \sum_\kappa
  \dim(\S_\kappa)$. } which significantly limits its
applicability. Nuclear norm regularization based methods, such as low
rank representation (LRR) \cite{Liu:ICML10} and low rank subspace
clustering (LRSC) \cite{Favaro:CVPR11,Vidal:PRL14}, also suffer from the same limitation \cite{Wang:NIPS13-LRR+SSC}.


To bridge the gap between the subspace preserving and connectedness properties,
\cite{Wang:NIPS13-LRR+SSC,Panagakis:PRL14,Fang:TKDE15} propose to use mixed norms. For example, the low rank sparse subspace clustering (LRSSC) method \cite{Wang:NIPS13-LRR+SSC}, which uses a mixed $\ell_1$ and nuclear norm regularizer, is shown to give a subspace preserving representation under conditions which are similar to but stronger than those of SSC. However, the justification for the improvements in connectivity given by LRSSC is merely experimental. Likewise, \cite{Panagakis:PRL14,Fang:TKDE15} propose to use a mixed $\ell_1$ and $\ell_2$ norm given by
\begin{equation}
r(\c) = \lambda \|\c\|_1 + \frac{1-\lambda}{2} \|\c\|_2 ^2,
\label{eq:r-l1_l2}
\end{equation}
where $\lambda \in [0,1]$ controls the trade-off between the two
regularizers. However, \cite{Panagakis:PRL14,Fang:TKDE15} do not
provide a theoretical justification for the benefits of
the method. Other subspace clustering regularizers studied in
\cite{Lu:ICCV13-TraceLasso} and \cite{Lai:ECCV14} use the trace
lasso \cite{Grave:NIPS11} and the $k$-support norm
\cite{Argyriou:NIPS12}, respectively. However,
no theoretical justification is provided in \cite{Lu:ICCV13-TraceLasso,Lai:ECCV14} for the benefit of their methods.

Another issue with the aforementioned methods \cite{Wang:NIPS13-LRR+SSC, Panagakis:PRL14, Fang:TKDE15, Lu:ICCV13-TraceLasso, Lai:ECCV14} is that they do not provide efficient algorithms to deal with large-scale datasets. To address this issue,  \cite{Chen:AAAI11} proposes to find the representation of $X$ by a few anchor points that are sampled from $X$ and then perform spectral clustering on the anchor graph. In \cite{Peng:CVPR13} the authors propose to cluster a small subset of the original data and then classify the rest of the data based on the learned groups. However, both of these strategies are suboptimal in that they sacrifice clustering accuracy for computational efficiency.


\myparagraph{Paper Contributions}
In this paper, we exploit a mixture of $\ell_1$ and $\ell_2$ norms to balance the subspace preserving and connectedness properties. Specifically, we use $r(\cdot)$ as in \eqref{eq:r-l1_l2} and $h(\e) = \frac{1}{2}\|\e\|_2^2$. The method is thus a combination of SSC and LSR and reduces to each of them when $\lambda = 1$ and $\lambda = 0$, respectively. In the statistics literature, the optimization program using this regularization is called \textit{Elastic Net} and is used for variable selection in regression problems \cite{Zou:JRSS05}. Thus we refer to this method as the Elastic Net Subspace Clustering (EnSC).
%
%
%
%

This work makes the following contributions:
\begin{enumerate}
\item We propose an efficient and provably correct active-set based algorithm for solving the elastic net problem. The proposed algorithm exploits the fact that the nonzero entries of the elastic net solution fall into an \emph{oracle region}, which we use to define and efficiently update an active set. The proposed update rule leads to an iterative algorithm which is shown to converge to the optimal solution in a finite number of iterations.

\item We provide theoretical conditions under which the affinity generated by EnSC is subspace preserving, as well as a clear geometric interpretation for the balance between the subspace-preserving and connectedness properties. Our conditions depend on a \emph{local} characterization of the distribution of the data, which improves over prior \emph{global} characterizations.

\item We present experiments on computer vision datasets that demonstrate the superiority of our method in terms of both clustering accuracy and scalability.
\end{enumerate}


\section{Elastic Net: Geometry and a New Algorithm}
\label{sec:elastic-net}


In this section, we study the elastic net optimization problem, and present a new active-set based optimization algorithm for solving it. Consider the objective function
\begin{equation}
	f(\c;~ \b, A) := \lambda \|\c\|_1 + \frac{1-\lambda}{2} \|\c\|_2^2 + \frac{\gamma}{2} \|\b - A \c\|_2^2,
	\label{eq:def-f}
\end{equation}
where $\b \in \RR^D$, $A = [\a_1, \cdots, \a_N] \in \RR^{D\times N}$, $\gamma > 0$, and $\lambda \in [0,1)$ (the reader is referred to the appendix for a study of the case $\lambda=1$). Without loss of
generality, we assume that $\b$ and $\{\a_j\}_{j=1}^N$
are normalized to be of unit $\ell_2$ norm in our analysis. The elastic net model then
computes
\begin{equation}	\label{eq:en}
	\c^*(\b, A) := \argmin_{\c} f(\c;~\b, A).
\end{equation}
We note that $\c^*(\b,A)$ is unique since $f(\c;~\b,A)$ is a strongly
convex function; we use the notation $\c^*$ in place of $\c^*(\b,A)$
when the meaning is clear.


In the next two sections, we present a geometric analysis of the
elastic net solution, and use this analysis to design an
active-set algorithm for efficiently solving~\eqref{eq:en}.

\subsection{Geometric structure of the elastic net solution}

\begin{figure*}
	\centering
	\subfigure[$\lambda = 1$]{\includegraphics[scale = 0.26]{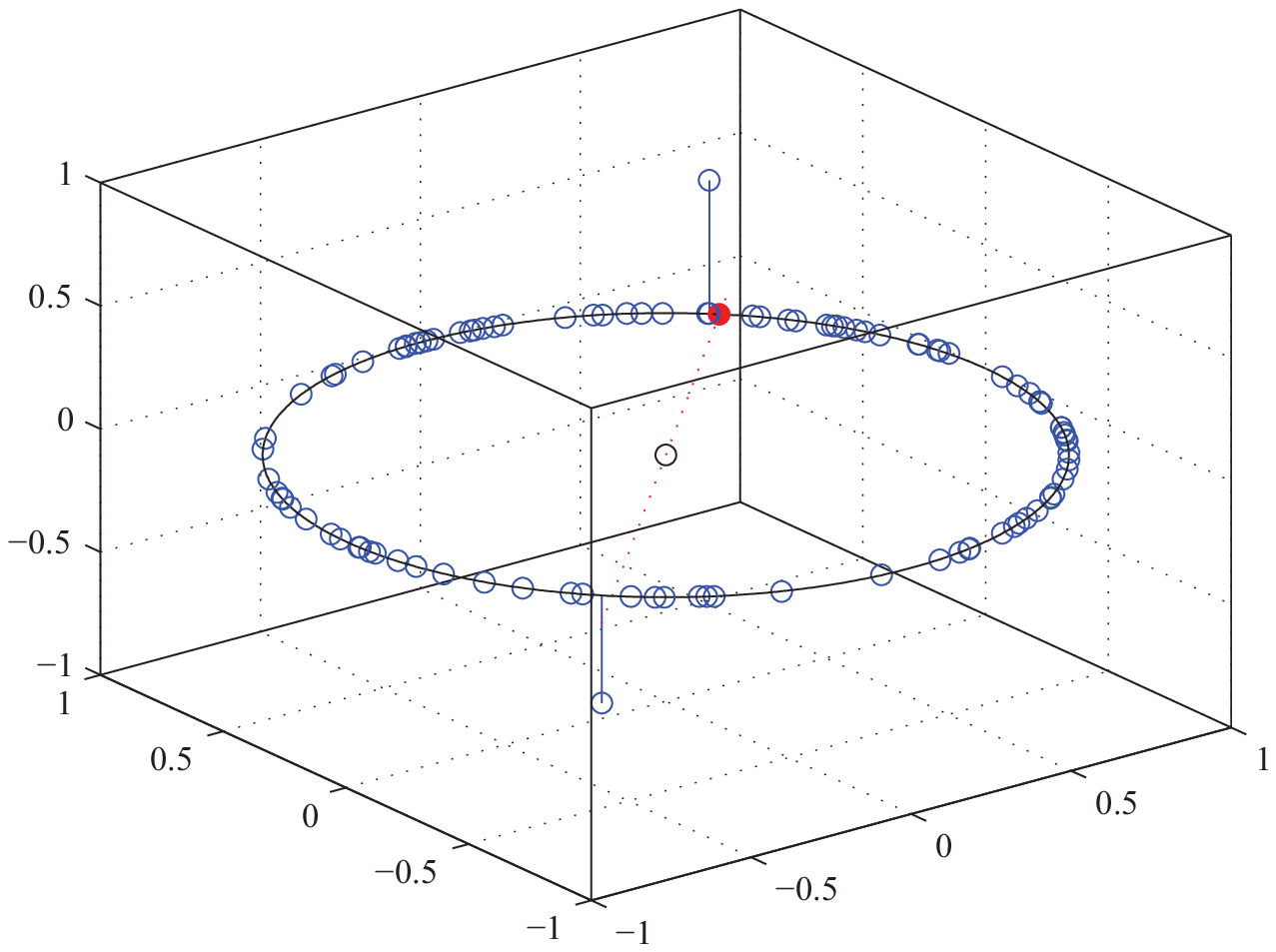}}
	~
	\subfigure[$\lambda = 0.9$]{\includegraphics[scale = 0.26]{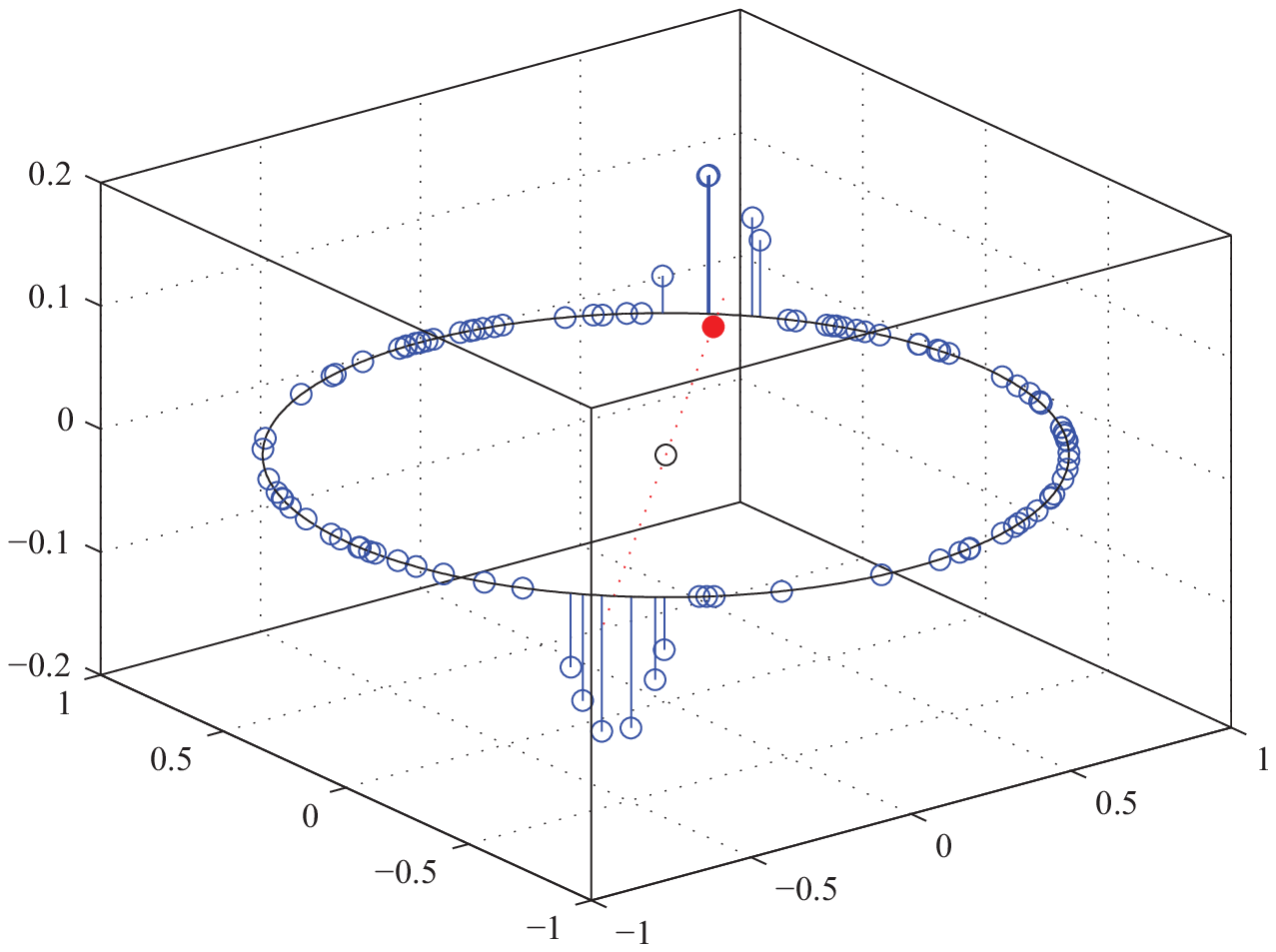}}
	~
	\subfigure[$\lambda = 0.3$]{\includegraphics[scale = 0.26]{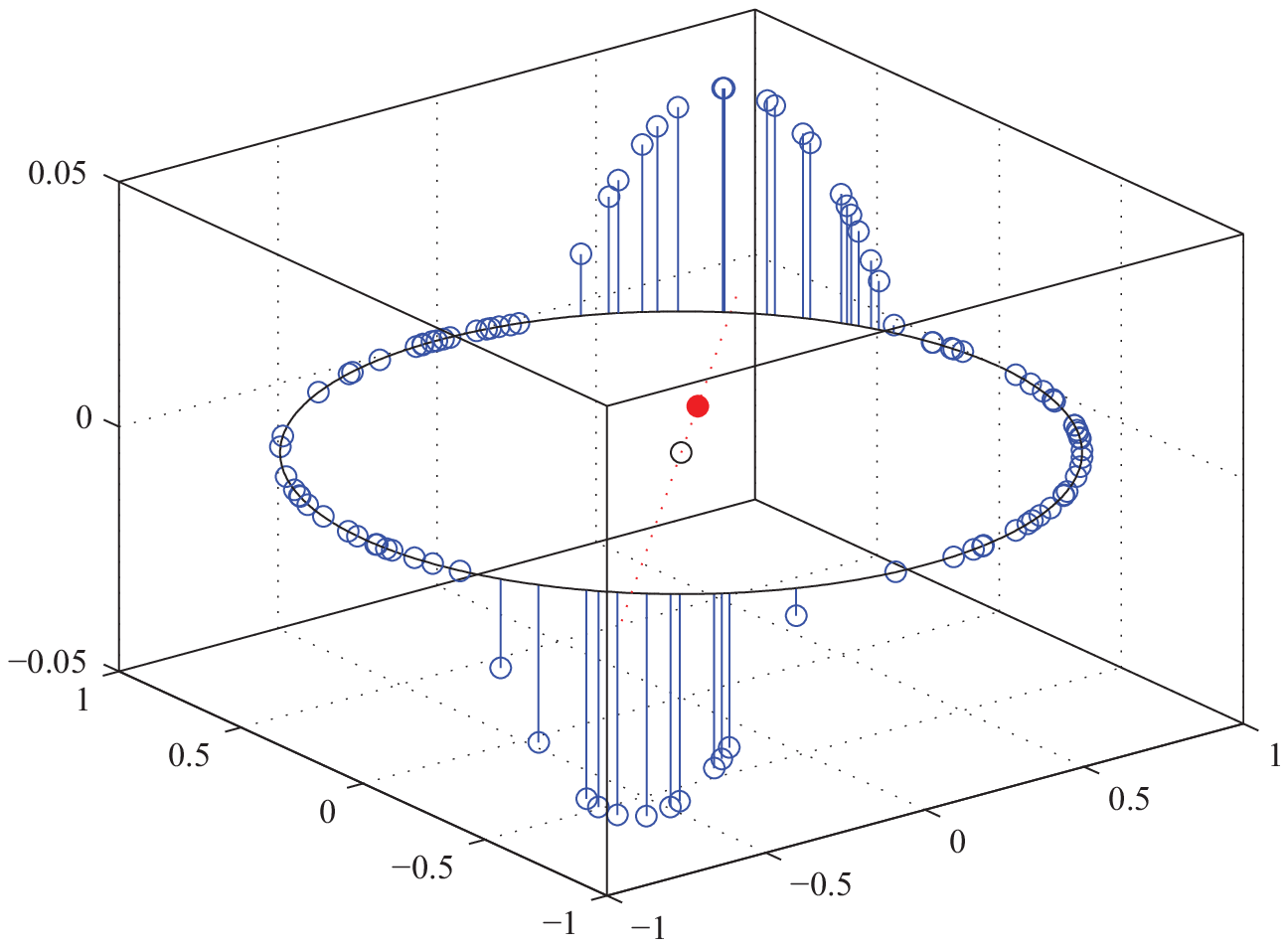}}
	~
	\subfigure[$\lambda = 0$]{\includegraphics[scale = 0.26]{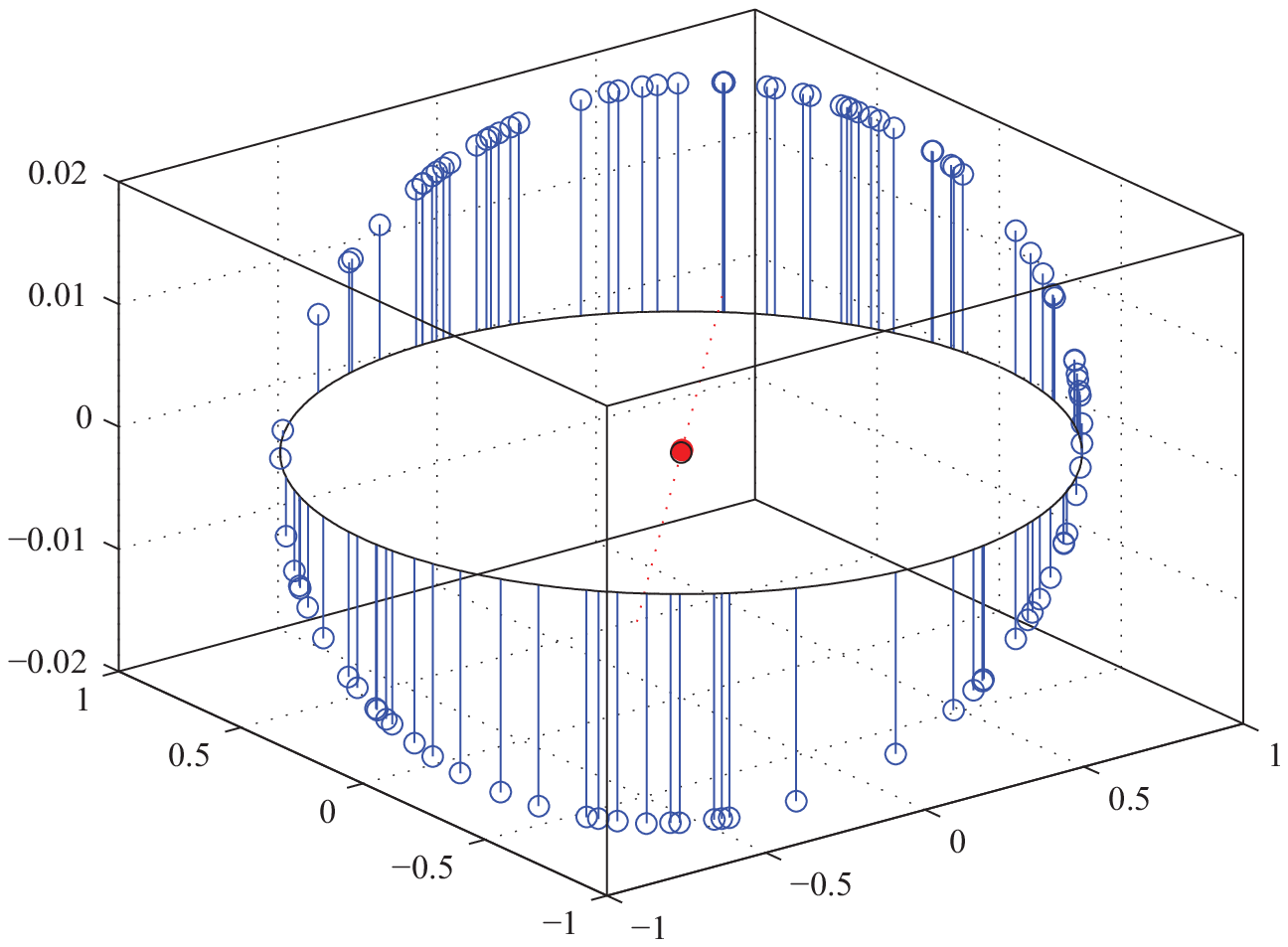}}
	\caption{Illustration of the structure of the solution $\c^*$
          for a data matrix $A$ containing 100 randomly generated points
          in $\RR^2$, which are shown as blue dots in the $x$-$y$ plane. The $z$
          direction shows the magnitude for each coefficient $\c^*_j$. The
          red dot represents the oracle point $\bfdelta(\b, A)$, with
          its direction denoted by the red dashed line. The value for $\gamma$
          is fixed at $50$, but the value for $\lambda$ varies as depicted.}
	\label{fig:geometry}
\end{figure*}
We first introduce the concept of an oracle point.

\begin{definition}[Oracle Point] \label{def:delta}
	The oracle point associated with the optimization problem
        \eqref{eq:en} is defined to be
	\begin{equation}
	\bfdelta(\b, A) := \gamma \cdot \big(\b - A \c^*(\b, A)\big).
	\label{eq:def-delta}
	\end{equation}
\end{definition}
\noindent
When there is no risk of confusion, we omit the dependency of the oracle point on $\b$ and $A$
and write $\bfdelta(\b, A)$~as~$\bfdelta$.

Notice that the oracle point is unique since $\c^*$ is unique, and
that the oracle point cannot be computed until the optimal
solution $\c^*$ has been computed.  
%
The next result gives a critical relationship involving the oracle point
that is exploited by our active-set method.
%
\begin{theorem}	\label{thm:geometry}
	The solution $\c^*$ to problem \eqref{eq:en} satisfies
	\begin{equation}
	(1-\lambda) \c^* = \T_\lambda(A^\transpose \bfdelta),
	\label{eq:geometry}
	\end{equation}
	where $\T_\lambda(\cdot)$ is the soft-thresholding operator
        (applied componentwise to $A^\transpose \bfdelta$) defined as
        $\T_\lambda(v) = sgn(v) (|v| - \lambda)$ if $|v| >
        \lambda$ and $0$ otherwise.
\end{theorem}




Theorem \ref{thm:geometry} shows that if the oracle point $\bfdelta$ is
known, the solution $\c^*$ can be written out directly.
Moreover, it follows from \eqref{eq:def-delta} and \eqref{eq:geometry} that $\bfdelta = \0$
if and only if $\b = \0$.

In Figure \ref{fig:geometry}, we depict a two dimensional example of
the solution to the elastic net problem \eqref{eq:en} for different
values of the tradeoff parameter $\lambda$. As
expected, the solution $\c^\ast$ becomes denser as $\lambda$
decreases. Moreover, as predicted by Theorem \ref{thm:geometry}, the
magnitude of the coefficient $c^*_j$ is a decaying function of the
angle between the corresponding dictionary atom $\a_j$ and the
oracle point $\bfdelta$ (shown in red). If $\a_j$ is far enough
from $\bfdelta$ such that $|\langle \a_j, \bfdelta \rangle| \leq \lambda$ holds true, then
the corresponding coefficient $c^*_j$ is zero. We call the region containing
the nonzero coefficients the \emph{oracle region}. We can formally define the oracle region by using the quantity
$\mu(\cdot, \cdot)$ to denote the coherence of two vectors, i.e.,
\begin{equation}
\mu(\v, \w) := \frac{|\langle \v, \w\rangle|}{\|\v\|_2 \|\w\|_2}.
\end{equation}

\begin{definition}[Oracle Region]	\label{def:Delta}
	The oracle region associated with the optimization problem
        \eqref{eq:en} is defined as
	\begin{equation}
	\Delta(\b, A) := \Big\{ \v \in \RR^D \!: \! \|\v\|_2 = 1, \, \mu(\v, \bfdelta) > \frac{\lambda}{\|\bfdelta\|_2}  \Big\}.\!\!
	\label{eq:def-Delta}
	\end{equation}
\end{definition}
The oracle region is composed of an antipodal pair of spherical caps of the unit ball
 of $\RR ^D$ that are located at the symmetric locations $\pm
 \bfdelta / \|\bfdelta\|_2$, both with an angular radius of $\theta = \arccos(
 \lambda / \|\bfdelta\|_2)$ (see Figure \ref{fig:Delta-illustration}).
 From the definition of the oracle region and Theorem
 \ref{thm:geometry}, it follows that $c_j^* \ne 0$ if and only if
 $\a_j \in \Delta(\b,A)$. In other words, the support of the solution
 $\c^*$ are those vectors $\a_j$ in the oracle region.

The oracle region also captures the behavior of the solution when columns from the matrix $A$ are removed or new columns are added. This provides the key insight into designing an active-set method for solving the optimization.


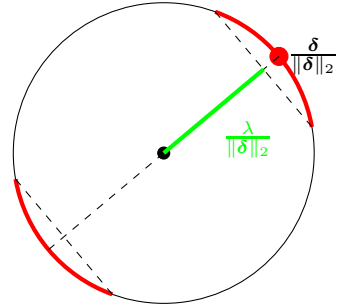
\begin{figure}
	\centering
	
	\def\Sx{1.2}
	\def\Sy{1.2}
	\def\angleL{10}
	\def\angleM{40}
	\def\angleR{70}
	\def\angleLs{190}
	\def\angleMs{220}
	\def\angleRs{250}
	
	\begin{tikzpicture}[scale = 2]
	\coordinate (0) at (0,0);
	\draw[black, fill = none] (0) circle [radius = 1];
	\draw [black, fill=black] (0) circle [radius=0.04];
	\draw [red, fill=red] (cos \angleM, sin \angleM) circle [radius=0.06];
	\node[black, right] at (cos \angleM, sin \angleM) {$\frac{\bfdelta}{\|\bfdelta\|_2}$};
	\draw[solid, ultra thick, red] (0) ([shift=(\angleL:1cm)] 0, 0) arc (\angleL:\angleR:1cm);
	\draw[solid, ultra thick, red] (0) ([shift=(\angleLs:1cm)] 0, 0) arc (\angleLs:\angleRs:1cm);
	\draw [dashed] (cos \angleL, sin \angleL) -- (cos \angleR, sin \angleR);
	\draw [dashed] (0, 0) -- (cos \angleM, sin \angleM);
	
	\draw [dashed] (cos \angleLs, sin \angleLs) -- (cos \angleRs, sin \angleRs);
	\draw [dashed] (0, 0) -- (cos \angleMs, sin \angleMs);
	\draw [green, ultra thick] (0, 0) -- (0.866 * cos \angleM, 0.866 * sin \angleM);
	\node[green, below right] at (0.433 * cos \angleM, 0.433 * sin \angleM) {$\frac{\lambda}{\|\bfdelta\|_2}$};
	\end{tikzpicture}
	\caption{The oracle region $\Delta(\b,A)$ is illustrated in
          red. Note that the size of the oracle region increases
          as the quantity $\lambda / \|\bfdelta\|_2$ decreases, and vice
          versa.}
	\label{fig:Delta-illustration}
\end{figure}


\begin{proposition}
For any $\b\in\RR^D$, $A \in \RR ^{D \times N}$ and $A' \in \RR ^{D \times N'}$, if no column of $A'$ is contained in $\Delta(\b, A) $,
then $\c^*(\b, [A, A']) = [\c^*(\b, A)^\transpose, \0_{N'
	\times 1}^\transpose] ^\transpose$.
	\label{thm:Delta-dynamic-1}
\end{proposition}


The interpretation for Proposition \ref{thm:Delta-dynamic-1} is that
the solution $\c^*(\b, A)$ does not change (modulo padding with
additional zeros) when new
columns are added to the dictionary $A$, as long as the new columns
are not inside the oracle region $\Delta(\b, A)$. From another perspective, $\c^*(\b,
[A, A'])$ does not change if one removes columns from the dictionary
$[A, A']$ that are not in the oracle region $\Delta(\b, [A, A'])$.

\begin{proposition}	\label{thm:Delta-dynamic-2}
For any $\b\in\RR^D$, $A \in \RR ^{D \times N}$ and $A' \in \RR ^{D \times N'}$, denote $\c^*(\b, [A, A']) = [\c_A^\transpose,
\c_{A'}^\transpose]^\transpose$.
If any column of $A'$ lies within $\Delta(\b, A)$ , then $\c_{A'}^\transpose \ne \0$.
\end{proposition}

This result means that the solution to the elastic net problem will
certainly be changed by adding new columns that lie
within the oracle region to the dictionary.

In the next section, we describe an efficient algorithm for
solving the elastic net problem \eqref{eq:en} that is based on the geometric
structure and concentration behavior of the solution.

\subsection{A new active-set algorithm}

Although the elastic net optimization problem \cite{Zou:JRSS05} has been recently introduced for subspace clustering in \cite{Fang:ICDM12,Fang:TKDE15,Panagakis:PRL14}, such prior work does not
provide an efficient algorithm that can handle large-scale datasets.  
\pagebreak
In fact, such prior work solves the
elastic net problem using existing algorithms that require calculations
involving the full data matrix $A$ (\eg, the accelerated proximal gradient
(APG) \cite{Beck2009} is used in \cite{Fang:ICDM12} and the linearized
alternating direction method (LADM) \cite{Lin:NIPS11} is used in
\cite{Panagakis:PRL14}).
%
Here, we propose to solve the elastic net problem \eqref{eq:en} with
an active-set algorithm that is more efficient than both APG and
LADM, and can handle large-scale datasets.  We call our new method
(see Algorithm~\ref{alg:main})
ORacle Guided Elastic Net solver, or ORGEN for short.

The basic idea behind ORGEN is to solve a sequence of
reduced-scale subproblems defined by an active set that is itself determined
from the oracle region. Let $T_k$ be the active set at iteration $k$.
Then, the next active set $T_{k+1}$ is selected to contain the
indices of columns that are in the oracle region $\Delta(\b,
A_{T_k})$, where $A_{T_k}$ denotes the submatrix of $A$ with
columns indexed by $T_k$. We use Figure
\ref{fig:alg-illustration} for a conceptual illustration. In Figure
\ref{fig:alg-illustration-1} we show the columns of $A$
that correspond to the active set $T_k$ by labeling the corresponding
columns of $A_{T_k}$ in
red. The oracle region $\Delta(\b, A_{T_k})$ is the union of the red
arcs in Figure \ref{fig:alg-illustration-2}. Notice that at the bottom
left there is one red dot that is not in $\Delta(\b, A_{T_k})$ and
thus must not be included in $T_{k+1}$, and two blue dots that are
not in $T_k$ but lie in the oracle region $\Delta(\b, A_{T_k})$ and
thus must be included in $T_{k+1}$. In Figure \ref{fig:alg-illustration-3}
we illustrate $T_{k+1}$ by green dots. This iterative procedure is terminated once
$T_{k+1}$ does not contain any new points, \ie, when $T_{k+1}\subseteq T_k$, at
which time $T_{k+1}$ is the support for $\c^*(\b, A)$.

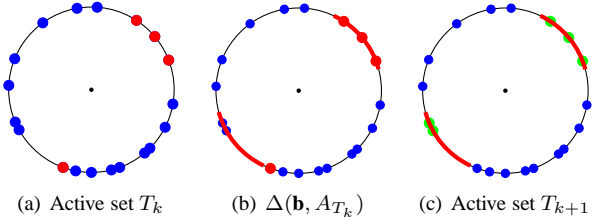
\begin{figure}[t]
	\centering
	\def\nPoint{19}
	\def\primalAngle
	{{21, 40, 57, 86, 100, 126, 158, 177, 203, 209, 250, 259, 270, 284, 290, 310, 315, 333, 350}}
	
	\def\Sx{1.2}
	\def\Sy{1.2}
	
	\subfigure[\label{fig:alg-illustration-1} Active set $T_k$] 
	{
		\begin{tikzpicture}[scale = 1.1]
		\coordinate (0) at (0,0);
		
		\draw[black, fill = none] (0) circle [radius = 1];
		\draw [black, fill=black] (0) circle [radius=0.02];
		\foreach \i in {1, 2, ..., \nPoint}
		{
			\def\angle{\primalAngle[\i-1]}
			\draw [blue, fill=blue] (cos \angle, sin \angle) circle [radius=0.06];
		}
		\foreach \i in {1, 2, 3, 11}
		{
			\def\angle{\primalAngle[\i-1]}
			\draw [red, fill=red] (cos \angle, sin \angle) circle [radius=0.06];
		}
		\end{tikzpicture}
	}
	~
%
	\subfigure[\label{fig:alg-illustration-2} $\Delta(\b, A_{T_k})$]
	{
		\begin{tikzpicture}[scale = 1.1]
		\coordinate (0) at (0,0);
		\def\dualAngle{39.60}
		\draw[black, fill = none] (0) circle [radius = 1];
		\draw [black, fill=black] (0) circle [radius=0.02];
		\foreach \i in {1, 2, ..., \nPoint}
		{
			\def\angle{\primalAngle[\i-1]}
			\draw [blue, fill=blue] (cos \angle, sin \angle) circle [radius=0.05];
		}
		\foreach \i in {1, 2, 3, 11}
		{
			\def\angle{\primalAngle[\i-1]}
			\draw [red, fill=red] (cos \angle, sin \angle) circle [radius=0.06];
		}
		\draw[solid, ultra thick, red] (0) ([shift=(15:1cm)] 0, 0) arc (15:65:1cm);
		\draw[solid, ultra thick, red] (0) ([shift=(195:1cm)] 0, 0) arc (195:245:1cm);
		\end{tikzpicture}
	}
	~
	\subfigure[\label{fig:alg-illustration-3} Active set $T_{k+1}$]
	{
		\begin{tikzpicture}[scale = 1.1]
		\coordinate (0) at (0,0);
		\def\dualAngle{39.60}
		\draw[black, fill = none] (0) circle [radius = 1];
		\draw [black, fill=black] (0) circle [radius=0.02];
		\foreach \i in {1, 2, ..., \nPoint}
		{
			\def\angle{\primalAngle[\i-1]}
			\draw [blue, fill=blue] (cos \angle, sin \angle) circle [radius=0.05];
		}
		\foreach \i in {1, 2, 3, 9, 10}
		{
			\def\angle{\primalAngle[\i-1]}
			\draw [green, fill=green] (cos \angle, sin \angle) circle [radius=0.06];
		}
		\draw[solid, ultra thick, red] (0) ([shift=(15:1cm)] 0, 0) arc (15:65:1cm);
		\draw[solid, ultra thick, red] (0) ([shift=(195:1cm)] 0, 0) arc (195:245:1cm);
		\end{tikzpicture}
	}
\caption{Conceptual illustration of the ORGEN algorithm. All the dots on the unit circle illustrate the dictionary $A$. (a) active set $T_k$ at step $k$, illustrated by red dots. (b) The oracle region $\Delta(\b, A_{T_k})$ illustrated by red arcs. (c) The new active set $T_{k+1}$ illustrated in green, which is the set of indices of points that are in $\Delta(\b, A_{T_k})$. }
\label{fig:alg-illustration}
	\vspace{-3mm}
\end{figure}

\begin{algorithm}
	\caption{ORacle Guided Elastic Net (ORGEN) solver}
	\label{alg:main}
	\begin{algorithmic}[1]
		\REQUIRE $A = [\a_1, \dots, \a_N] \in \RR ^{D \times N}$, $\b \in \RR ^D$, $\lambda$ and $\gamma$.
		\STATE Initialize the support set $T_0$ and set $k \leftarrow 0$. 
		\LOOP
		\STATE \label{step:solve-subproblem}Compute $\c^*(\b, A_{T_k})$ as in \eqref{eq:en} using any solver.
		\STATE Compute $\bfdelta(\b, A_{T_k})$ from $\c^*(\b, A_{T_k})$ as in \eqref{eq:def-delta}.
		\STATE \label{step:update-support} Active set update: $T_{k+1} \leftarrow \{ j: \a_j \in \Delta(\b, A_{T_k}) \}$.
		\STATE If $T_{k+1} \subseteq T_k$, terminate; otherwise set $k \leftarrow k+1$.
		\ENDLOOP
		\ENSURE A vector $\c$ such that $\c_{T_k} = \c^*(\b, A_{T_k})$ and zeros otherwise. Its support is $T_{k+1}$.
	\end{algorithmic}
\end{algorithm}

The next lemma helps explain why ORGEN converges.

\begin{lemma} \label{lem:decrease}
In Algorithm~\ref{alg:main}, if $T_{k+1} \nsubseteq T_k$, then
$$
f(\c^*(\b,A_{T_{k+1}});\b,A_{T_{k+1}}) < f(\c^*(\b,A_{T_k});\b,A_{T_k}).
$$
\end{lemma}

The following convergence result holds for ORGEN.

\begin{theorem}	\label{thm:alg-convergence}
	Algorithm \ref{alg:main} converges to the optimal solution $\c^*(\b, A)$ in a finite number of iterations.
\end{theorem}

The result follows from
Lemma~\ref{lem:decrease}, because it
implies that an active set can never be repeated.  Since there are
only finitely many distinct active sets, the algorithm must eventually
terminate with $T_{k+1} \subseteq T_k$.  The remaining part of the
proof establishes that if $T_{k+1} \subseteq T_k$, then
$\c^*(\b,A_{T_{k+1}})$ gives the nonzero entries of the solution.

ORGEN solves large-scale problems by solving a sequence of
reduced-size problems in step \ref{step:solve-subproblem} of
Algorithm~\ref{alg:main}. If the active set $T_k$ is small,
then step \ref{step:solve-subproblem} is a small-scale problem that
can be efficiently solved.
However, there is no procedure in Algorithm \ref{alg:main} that
explicitly controls the size of $T_k$. To address this concern, we
propose an alternative to
step \ref{step:update-support} in which only
a small number of new points---the ones most correlated with
$\bfdelta$---are added.  Specifically,
\begin{multline} \label{eq:alg-step5-alternative}
\text{\ref{step:update-support}'}\!: T_{k+1} = \{ j \in T_k : \a _j
\in \Delta(\b, A_{T_k})\} \cup S_k, 
\end{multline}
where
$S_k$ holds the indices of the largest $n$ entries in
$\{ |\a _j^\transpose\bfdelta(\b, A_{T_k})|: j \notin T_k, \, \a _j \in \Delta(\b, A_{T_k}) \}$;
%
ideally, $n$ should be chosen so that the size of $T_{k+1}$ is bounded by
a predetermined value $N_{\max}$ that represents the maximum
size subproblem that can be handled in step
\ref{step:solve-subproblem}. If $N_{\max}$ is chosen large
enough that the second set in the union in
\eqref{eq:alg-step5-alternative} is non-empty, then our convergence
result still holds.

\myparagraph{Initialization}
We suggest the following procedure for computing the initial active
set $T_0$.  First, compute the solution to
\eqref{eq:en} with $\lambda = 0$, which has a closed form solution and can be computed efficiently if the ambient dimension $D$ of the data is not too big.
Then, the $l$ largest entries
(in absolute value) of the solution for some pre-specified value $l$
are added to $T_0$.
Our experiments suggest that this strategy promotes fast
convergence of Algorithm~\ref{alg:main}.


\section{Elastic Net Subspace Clustering (EnSC)}
\label{sec:EnSC}


Although the elastic net has been recently introduced for subspace clustering in \cite{Panagakis:PRL14, Fang:TKDE15}, these works do not provide conditions under which the affinity is guaranteed to be subspace preserving or potential improvements in connectivity. In this section, we give conditions for the affinity to be subspace preserving and for the balance between the subspace-preserving and connectedness properties. To the best of our knowledge, this is the first time that such theoretical guarantees have been established.

We first formally define the subspace clustering problem.
\begin{problem}[\bf Subspace Clustering]
	\label{pro:problem}
	Let $X \in \RR^{D \times N}$ be a real-valued matrix whose columns are drawn from a union of $n$ subspaces of $\RR^D$, say $\bigcup_{\ell=1}^n \S_\ell$, where the dimension $d_\ell$ of the $\ell$-th subspace satisfies $d_\ell < D$ for $\ell=1,\dots,n$. The goal of subspace clustering is to segment the columns of $X$ into their representative subspaces.
\end{problem}

Let $X = [\x_1, \cdots, \x_N]$, where each $\x_j$ is assumed to be of unit norm. Using the same notation as for \eqref{eq:en}, the proposed EnSC computes $\c^*(\x_j, X_{-j})$ for each $\{\x_j\}_{j = 1}^N$, i.e.,
\begin{equation}
\c^*(\x_j, X_{-j}) = \argmin_{\c} f(\c; \x_j, X_{-j}),
\end{equation}
where $X_{-j}$ is $X$ with the $j$-th column removed.
In this section, we focus on a given vector, say $\x_j$.  We suppose that $\x_j \in \S_\ell$ for some $\ell$, and use $X_{-j} ^\ell$ to denote the submatrix of $X$ with columns from $S_\ell$ except that $\x_j$ is removed.
Since our goal is to use the entries of $\c^*(\x_j, X_{-j})$ to construct an affinity graph in which only points in the same subspace are connected, we desire the nonzero entries of $\c^*(\x_j, X_{-j})$ to be a subset of the columns $X_{-j} ^\ell$ so that no connections are built between points from different subspaces. If this is the case, we say that such a solution $\c^*(\x_j, X_{-j})$ is \emph{subspace preserving}. On the other hand, we also want the nonzero entries of $\c^*(\x_j, X_{-j})$ to be as dense as possible in $X_{-j} ^\ell$ so that within each cluster the affinity graph is well-connected\footnote{\modify{In fact, even when each cluster is well-connected, further improving connectivity within clusters is still beneficial since it enhances the ability of the subsequent step of spectral clustering in correcting any erroneous connections in the affinity graph \cite{vonLuxburg:StatComp2007,Vidal:GPCAbook}.}}. To some extent, these are conflicting goals: if the connections are few, it is more likely that the solution is subspace preserving, but the affinity graph of each cluster is not well connected. Conversely, as one builds more connections, it is more likely that some of them will be false, but the connectivity is improved.


In the next two sections, we give a geometric interpretation of the tradeoff between the subspace preserving and connectedness properties, and provide sufficient conditions for a representation to be subspace preserving.


\subsection{Subspace-preserving vs. connected solutions}



Our analysis is built upon the optimization problem $\min_{\c} f(\c; \x_j, X_{-j} ^\ell)$. Note that its solution is trivially subspace preserving since the dictionary $X_{-j}^\ell$ is contained in $\S _\ell$. We then treat all points from other subspaces as newly added columns to $X_{-j} ^\ell$ and apply Propositions \ref{thm:Delta-dynamic-1} and \ref{thm:Delta-dynamic-2}. We get the following geometric result.

\begin{lemma}	\label{thm:subspace-preserving-lemma}
	Suppose that $\x_j \in \S_\ell$.  Then, the vector $\c^*(\x_j, X_{-j})$ is subspace preserving if and only if $\x_k \notin \Delta(\x_j, X_{-j} ^\ell)$ for all $\x_k \notin \S_\ell$.
\end{lemma}

We illustrate the geometry implied by Lemma~\ref{thm:subspace-preserving-lemma} in Figure \ref{fig:geometry-SC}, where we assume $\S_\ell$ is a two dimensional subspace in $\RR^3$. The dictionary $X_{-j} ^\ell$ is represented by the blue dots in the plane and the oracle region $\Delta(\x_j, X_{-j} ^\ell)$ is denoted as the two red circles. The green dots are all other points in the dictionary. Lemma \ref{thm:subspace-preserving-lemma} says that $\c^*(\x_j, X_{-j})$ is subspace preserving if and only if all green dots lie outside of the red region.

\begin{figure}[t]
	\centering
	\def\nPoint{19}
	\def\primalAngle
	{{21, 30, 57, 86, 100, 126, 158, 177, 203, 209, 250, 259, 270, 284, 290, 310, 315, 333, 350}}
	\def\dualAngle{39.60}
	\def\Gamma{21.6}
	
	\tdplotsetmaincoords{65}{10}
	\begin{tikzpicture}[tdplot_main_coords, scale = 2.0]
	\def\R{1}
	
	\draw [black, fill=black] plot [mark=*, mark size=0.6] coordinates{(0, 0, 0)};
	\draw plot [domain = 0:360, samples = 90, variable = \i]
	(\R*cos \i, \R*sin \i, 0) -- cycle;
	\foreach \i in {0, 30,...,150}
	\draw [dotted] plot [domain = 0:360, samples = 60, variable = \j]
	(\R*cos \i*sin \j,\R*sin \i*sin \j, \R*cos \j);
	\foreach \j in {30, 60, ..., 150}
	\draw [dotted] plot [domain = 0:360, samples = 60, variable = \i]
	(\R*cos \i*sin \j,\R*sin \i*sin \j, \R*cos \j);
	\def\Sx{1.3}
	\def\Sy{1.7}
	\filldraw[draw=none,fill=gray!20, opacity=0.2]
	(-\Sx,-\Sy,0) -- (-\Sx, \Sy, 0) -- (\Sx, \Sy, 0) -- (\Sx, -\Sy, 0) -- cycle;
	\node[black] at (1.1, 1.1, 0) {$\S_\ell$};
	
	\draw [red, fill = red, fill opacity = 0.5]
	plot [domain = 0:360, samples = 40, variable = \i]
	(cos \dualAngle * cos \Gamma - sin \dualAngle * sin \Gamma * cos \i,
	sin \dualAngle * cos \Gamma + cos \dualAngle * sin \Gamma * cos \i,
	sin \Gamma * sin \i) -- cycle;
	\draw [red, fill = red, fill opacity = 0.5]
	plot [domain = 0:360, samples = 40, variable = \i]
	(- cos \dualAngle * cos \Gamma + sin \dualAngle * sin \Gamma * cos \i,
	- sin \dualAngle * cos \Gamma - cos \dualAngle * sin \Gamma * cos \i,
	sin \Gamma * sin \i) -- cycle;
	
	\foreach \j in {1, 2, ..., \nPoint}
	{
		\def\angle{\primalAngle[\j-1]}
		\draw [blue, fill=blue] plot [mark=*, mark size=0.6] coordinates{(cos \angle, sin \angle, 0)};
	}
	\foreach \j/\i in {30/30, 20/120, 70/300, 75/200, 120/35, 170/330, 150/120, 160/45}
	{
		\draw [green, fill=green] plot [mark=*, mark size=0.6] coordinates{(cos \i*sin \j,sin \i*sin \j, cos \j)};
		\draw [dashed, green] (0,0,0) -- (cos \i*sin \j,sin \i*sin \j, cos \j);
	}
	\end{tikzpicture}
	\caption{The structure of the solution for an example in $\RR^3$ associated with a point $\x_j$ (not shown) that lies in the 2-dimensional subspace $\S_\ell$.  The blue dots illustrate the columns of $X_{-j} ^\ell$, the union of the two red regions is the oracle region $\Delta(\x_j, X_{-j} ^\ell)$, and the green points are vectors from other subspaces.}
	\label{fig:geometry-SC}
\end{figure}
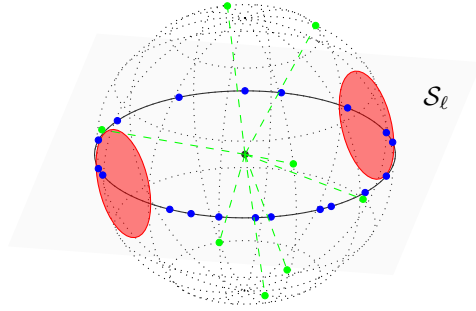

To ensure that a solution is subspace preserving one desires a small oracle region, while to ensure connectedness one desires a large oracle region. These facts again highlight the trade-off between these two properties.
Recall that the elastic net balances $\ell_1$ regularization (promotes sparse solutions) and $\ell_2$ regularization (promotes dense solutions). Thus, one should expect that the oracle region will decrease in size as $\lambda$ is increased from $0$ towards $1$. Theorem~\ref{thm:bound-region} formalizes this claim, but first we need the following definition that characterizes the distribution of the data in $X_{-j} ^\ell$.

\begin{definition}[inradius]	\label{def:inradius}
The inradius of a convex body $\P$ is the radius $r(\P)$ of the largest $\ell_2$ ball inscribed in $\P$.
\end{definition}

To understand the next result, we comment that the size of the oracle region $\Delta(\x_j, X_{-j} ^\ell)$ is controlled by the quantity $\lambda/\|\bfdelta(\x_j, X_{-j} ^\ell)\|_2$ as depicted in Figure \ref{fig:Delta-illustration}. 
\begin{theorem}	\label{thm:bound-region}
	If $\x_j\in\S_\ell$, then
	\begin{equation}
	\frac{\lambda}{\|\bfdelta(\x_j, X_{-j} ^\ell)\|_2} \ge  \frac{r_j ^2 }{r_j + \frac{1-\lambda}{\lambda}},
	\label{eq:bound-region}
	\end{equation}
	where $r_j$ is the inradius of the convex hull of the symmetrized points in $X_{-j} ^\ell$, i.e.,
	\begin{equation}
		r_j := r(\conv\{ \pm\x_k: \x_k \in \S_\ell \ \text{and} \ k \neq j \}).
		\label{eq:def-r_j}
	\end{equation}
We define the right-hand-side of \eqref{eq:bound-region} to be zero
when $\lambda=0$.
\end{theorem}

The above theorem allows us to determine an upper bound for the size of the oracle region. This follows since a lower bound on the size of $\lambda / \|\bfdelta(\x_j, X_{-j} ^\ell)\|_2$ implies an upper bound on the size of the oracle region (see \eqref{eq:def-Delta} and Figure~\ref{fig:Delta-illustration}).
Also notice that
the right hand side of \eqref{eq:bound-region} is in the range $[0, r_j)$ and is monotonically increasing with $\lambda$. Thus, it provides an upper bound on the area of the oracle region, which decreases as $\lambda$ increases. This highlights that the trade-off between the subspace-preserving and connectedness properties is controlled by $\lambda$.

\newcommand{\m}{\phantom{-}}

\begin{remark}
It would be nice if $\lambda / \|\bfdelta(\x_j, X_{-j} ^\ell)\|_2$ was increasing as a function of $\lambda$ (we already know that its lower bound given in Theorem~\ref{thm:bound-region} is increasing in $\lambda$). However,
one can show using the data $\x_j = [0.22, 0.72, 0.66]^\transpose$,
	\begin{equation}
	X_{-j} ^\ell = \left[
	\begin{array}{cccc}
	-0.55 & -0.82 & -0.05 & 0.22 \\
	\m 0.22  & \m0.57 & \m0.84 & 0.78 \\
	-0.80 & \m0.00 & \m0.55 & 0.58
	\end{array}
	\right],
	\end{equation}
and parameter choice $\gamma = 10$, that $\lambda/\|\bfdelta\|$ (with $\lambda = 0.88$) is larger than $\lambda/\|\bfdelta\|$ (with $\lambda = 0.95$).
\end{remark}

%
%

\subsection{Conditions for a subspace-preserving solution}

A sufficient condition for a solution to be subspace preserving is obtained by combining the geometry in Lemma \ref{thm:subspace-preserving-lemma} with the bound on the size of the oracle region implied by Theorem \ref{thm:bound-region}.

\begin{theorem}	\label{thm:subspace-preserving-condition-inradius}
	Let $\x_j \in\S_\ell$, $\bfdelta_j = \bfdelta(\x_j, X_{-j} ^\ell)$ be the oracle point, and $r_j$ be the inradius characterization of $X_{-j} ^\ell$ as given by  \eqref{eq:def-r_j}. Then, $\c^*(\x_j, X_{-j})$ is subspace preserving if
	\begin{equation}	\label{eq:subspace-preserving-condition-inradius}
	\max_{k: \x_k \notin \S_\ell}\mu(\x_k, \bfdelta_j) \le \frac{r_j^2}{r_j + \frac{1-\lambda}{\lambda}}.
	\end{equation}	
\end{theorem}

Notice that in Theorem~\ref{thm:subspace-preserving-condition-inradius} the quantity $\bfdelta_j$ is determined from $X^\ell_{-j}$ and that it lies within the subspace $\S _\ell$ by definition of $\bfdelta(\x_j,X_{-j}^\ell)$. Thus the left-hand-side of \eqref{eq:subspace-preserving-condition-inradius} characterizes the separation between the oracle point---which is in $\S _\ell$---and the set of points outside of $\S _\ell$. On the right-hand-side, $r_j$ characterizes the distribution of points in $X_{-j} ^\ell$. In particular, $r_j$ is large when points are well spread within $\S _\ell$ and not skewed toward any direction. Finally, note that the right-hand-side of \eqref{eq:subspace-preserving-condition-inradius} is an increasing function of $\lambda$, showing that the solution is more likely to be subspace preserving if more weight is placed on the $\ell_1$ regularizer relative to the $\ell_2$ regularizer.

Theorem~\ref{thm:subspace-preserving-condition-inradius} has a close relationship to the sufficient condition for SSC to give a subspace preserving solution (the case $\lambda = 1$) \cite{Soltanolkotabi:AS13}. Specifically, \cite{Soltanolkotabi:AS13} shows that if $\max_{k: \x_k \notin \S_\ell}\mu(\x_k, \bfdelta_j) < r_j$, then SSC gives a subspace preserving solution. We can observe that condition \eqref{eq:subspace-preserving-condition-inradius} approaches the condition for SSC as $\lambda \rightarrow 1$.

The result stated in Theorem \ref{thm:subspace-preserving-condition-inradius} is a special case of the following more general result.

\begin{theorem}	\label{thm:subspace-preserving-condition}
	Let $\x_j \in \S_\ell$, $\bfdelta_j = \bfdelta(\x_j, X_{-j} ^\ell)$ be the oracle point, and $\kappa_j = \max_{k \ne j, \x_k \in \S_\ell} \mu(\x_k, \bfdelta_j)$ be the coherence of $\bfdelta_j$ with its nearest neighbor in $X_{-j} ^\ell$. Then, the solution $\c^*(\x_j, X_{-j})$ is subspace preserving if
	\begin{equation}
	\max_{k: \x_k \notin \S_\ell}\mu(\x_k, \bfdelta_j) \le \frac{\kappa_j^2}{\kappa_j + \frac{1-\lambda}{\lambda}}.
	\label{eq:subspace-preserving-condition}
	\end{equation}
\end{theorem}

The only difference between this result and that in Theorem \ref{thm:subspace-preserving-condition-inradius} is that $\kappa_j$ is used instead of $r_j$ for characterizing the distribution of points in $X_{-j}^\ell$. We show in Lemma \ref{thm:inradius} that $r_j \le \kappa_j$, which makes Theorem \ref{thm:subspace-preserving-condition} more general than Theorem \ref{thm:subspace-preserving-condition-inradius}. Geometrically, $r_j$ is large if the subspace $\S _\ell$ is well-covered by $X _j ^\ell$, while $\kappa_j$ is large if the neighborhood of the oracle closest to $\bfdelta_j$ is well-covered, \ie, there is a point in $X _{-j}^\ell$ that is close to $\bfdelta_j$. \modify{Thus, while the condition in Theorem \ref{thm:subspace-preserving-condition-inradius} requires each subspace to have global coverage by the data, the condition in Theorem \ref{thm:subspace-preserving-condition} allows the data to be biased, and only requires a local region to be well-covered. In addition, condition \eqref{eq:subspace-preserving-condition} can be checked when the membership of the data points is known. This advantage allows us to check the tightness of the condition \eqref{eq:subspace-preserving-condition}, which is studied in more details in the appendix. In contrast, condition \eqref{eq:subspace-preserving-condition-inradius} and previous work on SSC \cite{Soltanolkotabi:AS13,Wang-Xu:ICML13} use the inradius $r_j$, which is generally NP-hard to calculate \cite{Soltanolkotabi:AS13,Wang:NIPS13-LRR+SSC}.}


\section{Experiments}

\subsection{ORGEN on synthetic data}

We conducted synthetic experiments to illustrate the computational
efficiency of the proposed algorithm ORGEN. Three popular
solvers are exploited: the regularized feature sign search (RFSS) is
an active set type method \cite{Jin:IP09};  the LASSO version of the
LARS algorithm \cite{Efron:AS04} that is implemented in the sparse modeling
software (SPAMS); and the gradient projection for sparse reconstruction
(GPSR) algorithm proposed in \cite{Figueiredo:STSP07}. These three
solvers are used to solve the subproblem in step
\ref{step:solve-subproblem} of ORGEN, resulting in three
implementations of ORGEN. We also used the three solvers as stand-alone
solvers for comparison purposes.

In all experiments, the vector $\b$ and columns of $A$ are all generated independently and uniformly at random on the unit sphere of $\RR^{100}$. The results are averages over $50$ trials.


In the first experiment, we test the scaling behavior of ORGEN by
varying $N$; the results are shown in Figure \ref{fig:alg_N_time}. We
can see that our active-set scheme improves the computational efficiency of
all three solvers. Moreover, as $N$ grows, the improvement becomes
more significant.

\begin{figure*}[t]
	\centering
	\subfigure[\label{fig:alg_N_time} Running time versus $N$]{\includegraphics[scale = 0.37]{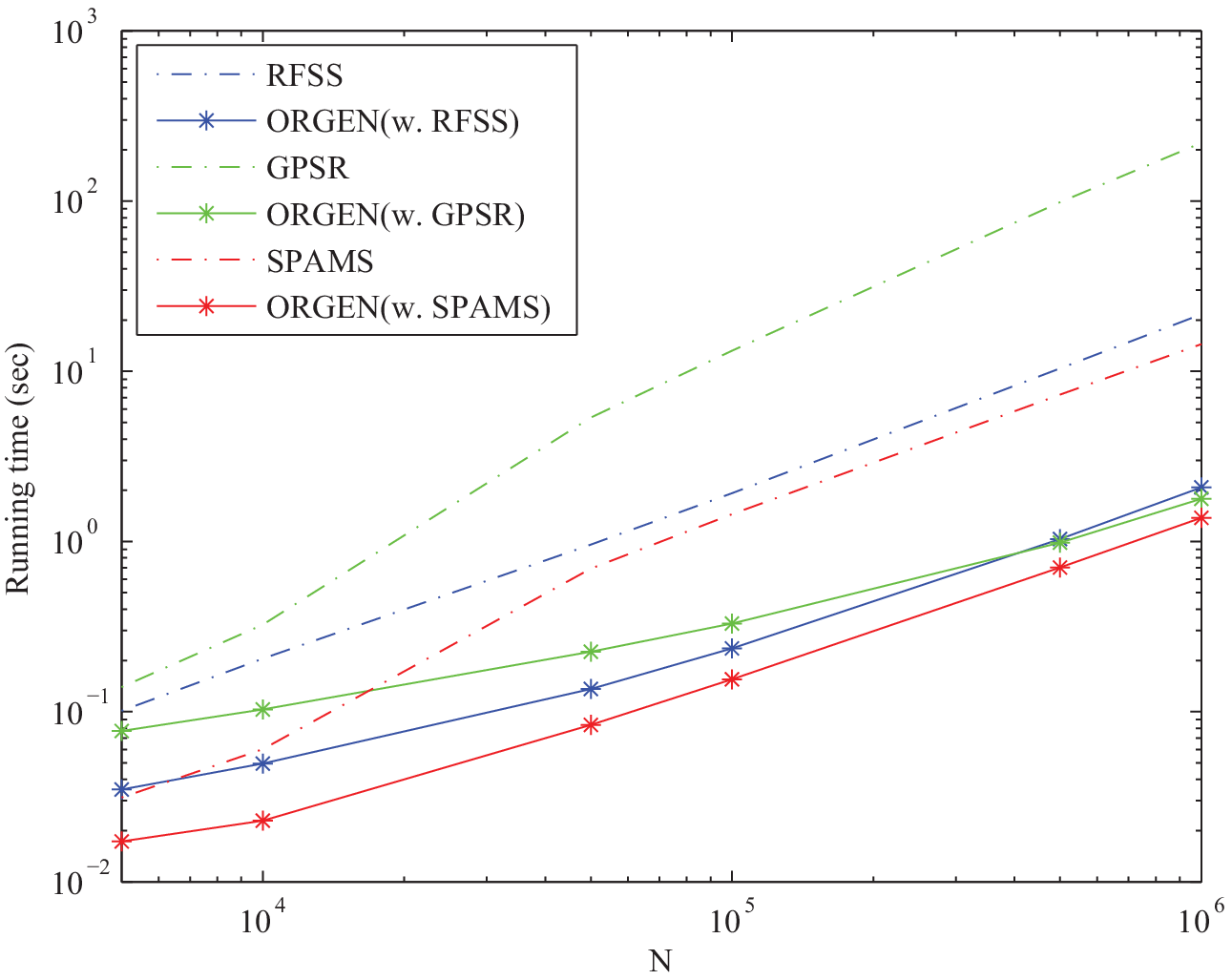}}
	~~~~~~
	\subfigure[\label{fig:alg_lambda_time} Running time versus $\lambda$]{\includegraphics[scale = 0.37]{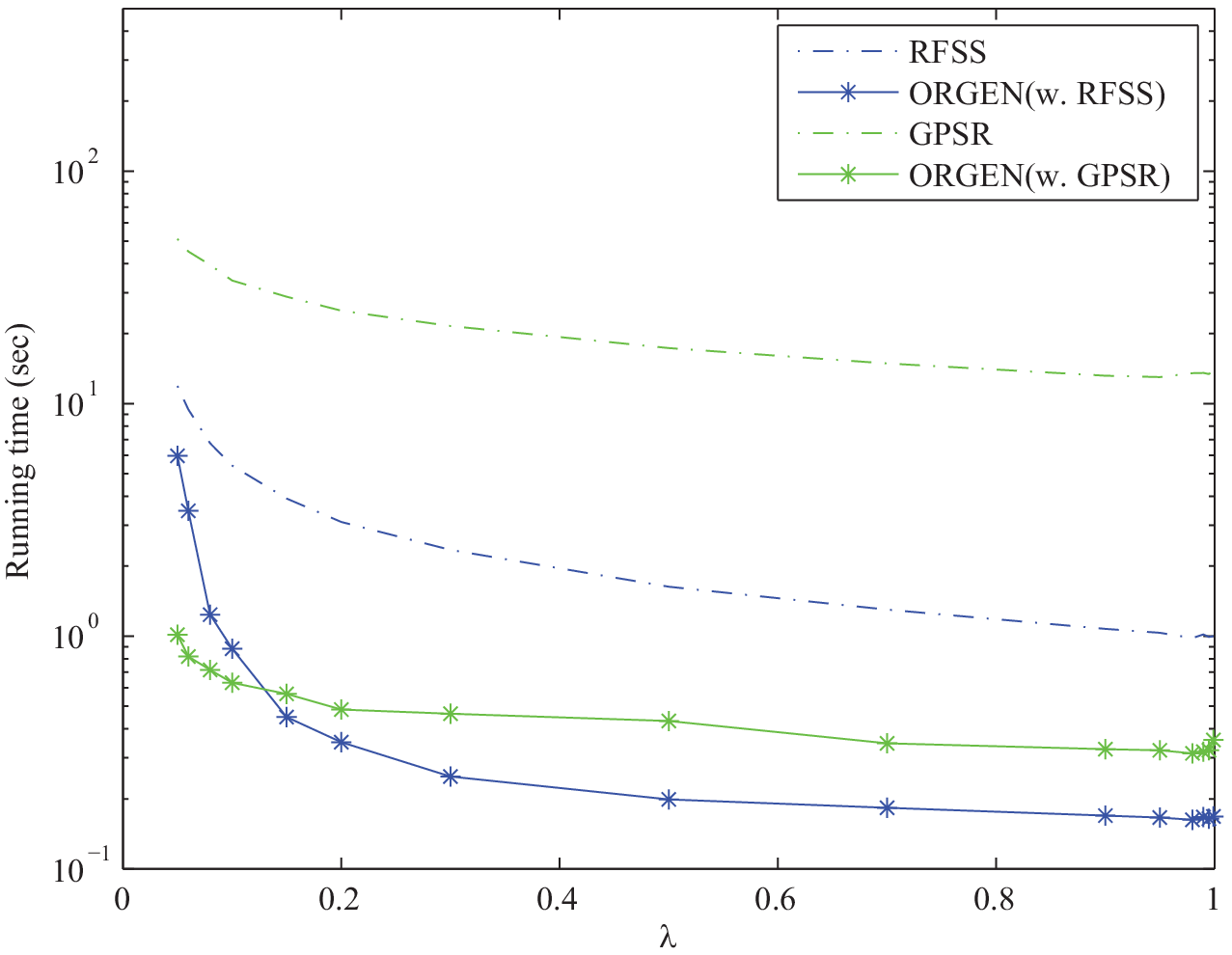}}
	~~~~~~
	\subfigure[\label{fig:alg_lambda_nnz} Sparsity versus $\lambda$]{\includegraphics[scale = 0.37]{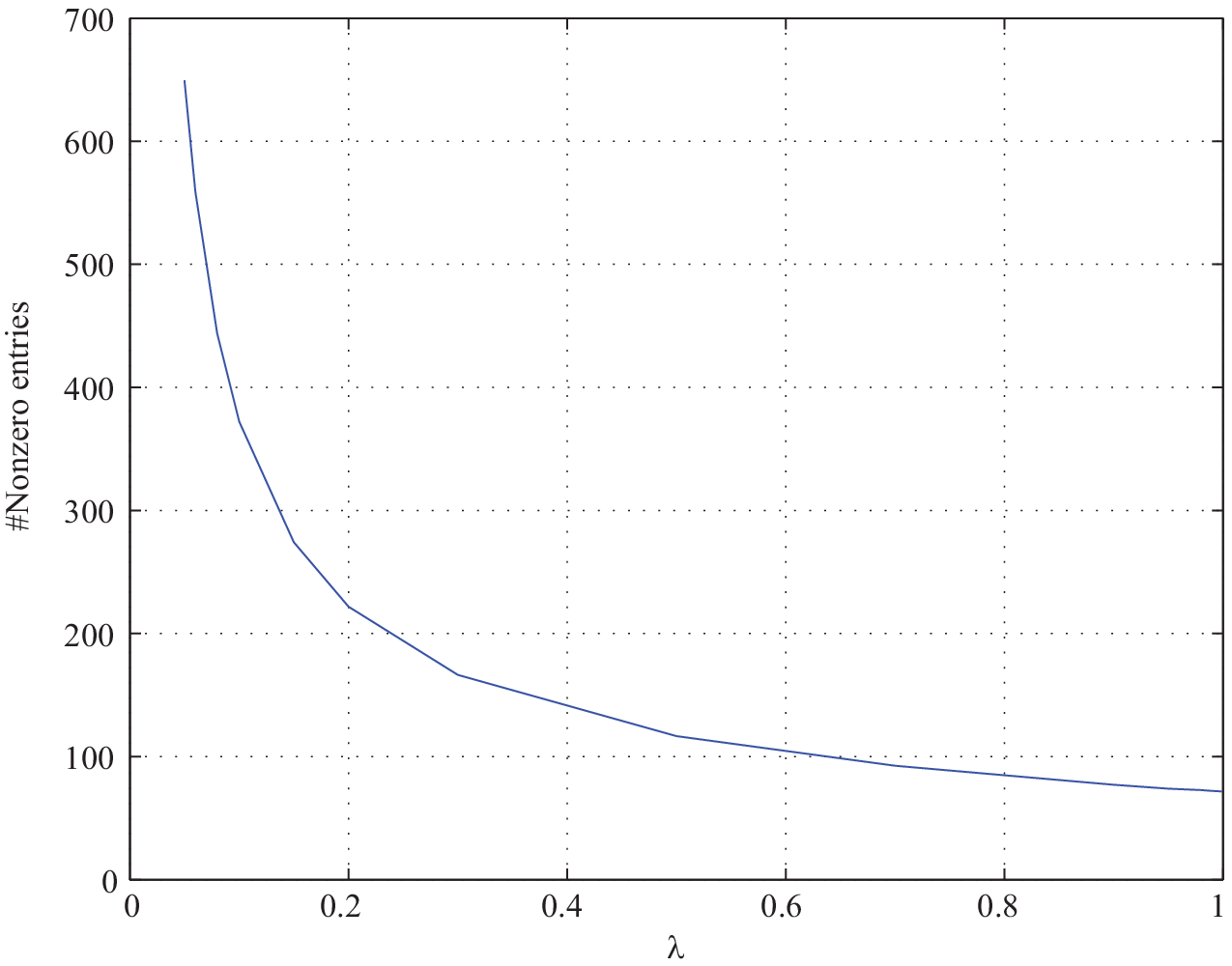}}
	\label{fig:alg_lambda}
	\caption{Performance with varying $N$ and
          $\lambda$: (a) $\lambda = 0.9$ and $N \in [5000,
            10^6]$; and (b, c) $N = 100,000$ and $\lambda\in [0.05, 0.999]$.}
\end{figure*}

Next, we test the performance of ORGEN for various values of the
parameter $\lambda$ that controls the tradeoff between the subspace preserving and connectedness properties; the running times and sparsity level are shown in Figures
\ref{fig:alg_lambda_time} and \ref{fig:alg_lambda_nnz}, respectively.
The performance of SPAMS is not reported since it
performs poorly even for moderately small values of $\lambda$.  For all
methods, the computational efficiency decreases as $\lambda$ becomes
smaller. For the two versions of ORGEN, this is expected since the
solution becomes denser as $\lambda$ becomes smaller (see Figure
\ref{fig:alg_lambda_nnz}).  Thus the active sets become
larger, which leads directly to larger and more time consuming
subproblems in step \ref{step:solve-subproblem}.

\begin{table}
	\centering
	\caption{Dataset information.}
	\label{tbl:dataset}
	\begin{tabular}{@{\;}c@{\;}|@{\;}c@{\;\;}c@{\;\;}c@{\;}}
		\hline
		& $N$ ($\#$data) & $D$ (ambient dim.) & $n$ ($\#$groups) \\ \hline
		Coil-100 &      7,200       &         1024          &        100         \\ \hline
		PIE    &      11,554      &         1024          &         68         \\ \hline
		MNIST   &      70,000      &          500          &         10         \\ \hline
		CovType  &     581,012      &          54           &         7          \\ \hline
	\end{tabular}
	\vspace{-2.5mm}
\end{table}


\subsection{EnSC on real data}


In this section, we use ORGEN to solve the optimization
problems arising in EnSC, where each subproblem in step
\ref{step:solve-subproblem} is solved using the RFSS method.
To compute the coefficient vectors $\c^*(\x_j, X_{-j})$, the parameter $\lambda$ is set to be the same for all $j$, while the parameter
$\gamma$ is set as $\gamma = \alpha \gamma_0$ where $\alpha > 1$ is a hyperparameter and $\gamma_0$ is the smallest value of $\gamma$ such that $\c^*(\x_j, X_{-j})$ is nonzero.
The algorithm is run for at most
$2$ iterations, as we observe that this is sufficient for the purpose
of subspace clustering and that subsequent iterations do not boost
performance. We measure clustering performance by clustering
accuracy, which is calculated as the best matching rate between the
label predicted by the algorithm and that of the ground truth.

\myparagraph{Datasets}
We test our method on the four datasets presented in Table \ref{tbl:dataset}.
The Coil-100 dataset \cite{Nene:1996-coil} contains $7,\!200$ gray-scale images of $100$
different objects. Each object has $72$ images taken at pose intervals
of $5$ degrees, with the images being of size $32 \times 32$. The PIE dataset \cite{Gross:IVC-10}
contains images of the faces of $68$ people taken under $13$ different poses, 43 different
illuminations, and 4 different expressions. In the experiments, we use
the five near frontal
poses and all images under different illuminations and
expressions. Each image is manually cropped and normalized to $32
\times 32$ pixels. The MNIST dataset \cite{LeCun:1998} contains $70,\!000$
images of handwritten digits $0$--$9$. For each image, we extract a
feature vector of dimension $3,\!472$ via the scattering convolution
network \cite{Bruna:PAMI13}, and then project to dimension $500$ using
PCA. Finally, the Covtype database\footnote{http://archive.ics.uci.edu/ml/datasets/Covertype} has been collected to predict forest cover type from $54$
cartographic variables.

\begin{table*}[t]
	\centering
	\caption{Performance of different clustering algorithms. The
          running time includes the time for computing the affinity matrix
          and for performing spectral clustering. The sparsity is the  number of nonzero coefficients in each representation $\c_j$ averaged over $j=1, \cdots, N$. The value ``M'' means
          that the memory limit of 16GB was exceeded, and the
          value ``T'' means that the time limit of
          seven days was reached. }
	\label{tbl:subspace-real}
	\begin{tabular}{c||cccc|cc|cc|c}
		         &  TSC  &  OMP  &  NSN  & SSC-SPAMS & SSC-ADMM & LRSC  & ENSC  &  KMP  &  EnSC-ORGEN   \\ \hline
		\multicolumn{10}{l}{\textsl{Clustering accuracy (\%)}}                                           \\ \hline
		Coil-100 & 61.32 & 33.64 & 50.32 &   53.75   &  57.10   & 55.76 & 51.11 & 61.97 & \textbf{69.24} \\ \hline
		  PIE    & 22.15 & 11.28 & 35.02 &   39.05   &  41.94   & 46.65 & 21.40 & 16.55 & \textbf{52.98} \\ \hline
		 MNIST   & 85.00 & 46.84 & 85.82 &   92.46   &    M     &   M   &   M   &   M   & \textbf{93.79} \\ \hline
		CovType  & 35.45 & 48.76 & 38.04 &     T     &    M     &   M   &   M   &   M   & \textbf{53.52} \\ \hline
		\multicolumn{10}{l}{\textsl{Running time (min.)}}                                                \\ \hline
		Coil-100 &   2   &   2   &  11   &    16     &   127    &   3   &   8   &  63   &       3        \\ \hline
		  PIE    &   3   &   8   &  25   &    67     &   412    &  12   &  25   &  361  &       13       \\ \hline
		 MNIST   &  30   &  24   &  298  &   1350    &    -     &   -   &   -   &   -   &       28       \\ \hline
		CovType  &  999  &  783  & 3572  &     -     &    -     &   -   &   -   &   -   &      1452      \\ \hline
		\multicolumn{10}{l}{\textsl{Sparsity}}                                                           \\ \hline
		Coil-100 &   4   &   2   &  18   &    7.0    &   5.4    & 7199  & 7199  &  20.9   &      6.3       \\ \hline
		  PIE    &   8   &  22   &  17   &   20.4    &   28.5   & 11553 & 11553 &  30.0   &      82.4      \\ \hline
		 MNIST   &   8   &  10   &  12   &   25.4    &    -     &   -   &   -   &   -   &      26.6      \\ \hline
		CovType  &  20   &  15   &  10   &     -     &    -     &   -   &   -   &   -   &      34.9      \\ \hline
	\end{tabular}
	\vspace{-0.8mm}
\end{table*}

\myparagraph{Methods}
We compare our method with several state-of-the-art subspace
clustering methods that may be categorized into three groups. The first
group contains
TSC~\cite{Heckel:arxiv13}, OMP~\cite{Dyer:JMLR13},
NSN~\cite{Park:NIPS14}, and SSC~\cite{Elhamifar:CVPR09}. TSC is a variant of the $k$-nearest neighbors method,
OMP and
NSN are two sparse greedy methods, and SSC is a convex optimization
method. These algorithms build sparse affinity
matrices and are computationally efficient, and therefore can perform large-scale
clustering. For TSC and NSN we use the code provided by
the respective authors. We note that the code may not be optimized for
computational efficiency considerations. For OMP, we use our
implementation, which has been optimized for subspace clustering.
For SSC we use the SPAMS solver described in the
previous section.

The second group consists of LRSC and SSC (with a different solver). We use the code provided by their respective authors, which uses the Alternating Direction Method of Multipliers (ADMM) to solve the optimization
problems. To distinguish the two versions of SSC, we refer to this one as SSC-ADMM and to the previous one as SSC-SPAMS.

The final group consists of ENSC~\cite{Panagakis:PRL14}
and KMP~\cite{Lai:ECCV14}, and are the closest in spirit to our method.
Our method and ENSC both balance the $\ell_1$ and
$\ell_2$ regularizations, but ENSC uses $h(\e) = \|\e\|_1$ to penalize
the noise (see \eqref{eq:self-expression}) and the linearized
alternating direction method to minimize their objective. In KMP,
the k-support norm is used to blend the $\ell_1$ and $\ell_2$
regularizers. We implemented ENSC and KMP according to
the descriptions in their original papers.

\myparagraph{Results}
To the best of our knowledge, a comparison of all these methods on large scale datasets has not been reported in prior work.
Thus, we run all experiments and tune the parameters for each method to give the best
clustering accuracy. The results are reported in Table
\ref{tbl:subspace-real}.

We see that our proposed method achieves the best clustering performance on every dataset. Our method is also among the most efficient in terms of computing time. The methods SSC-ADMM, ENSC, LRSC and KMP cannot handle large-scale data because they
perform calculations over the full data matrix and put the entire kernel matrix $X^\transpose X$ in memory, which is infeasible for large datasets. The method of SSC-SPAMS uses an active set method that can deal with massive data, however, it is computationally much less efficient than our solver ORGEN.


For understanding the advantages of our method, in Table \ref{tbl:subspace-real} we report the sparsity of the representation coefficients, which is the number of nonzero entries in $\c_j$ averaged over all $j=1, \ldots, N$. For TSC, OMP and NSN, the sparsity is directly provided as a parameter of the algorithms. For SSC and our method EnSC-ORGEN, the sparsity is indirectly controlled by the parameters of the models. We can see that our method usually gives more nonzero entries than the sparsity based methods of TSC, OMP, NSN, and SSC. This shows the benefit of our method: while the number of correct connections built by OMP, NSN and SSC are in general upper-bounded by the dimension of the subspace, our method does not have this limit and is capable of constructing more correct connections and producing well-connected affinity graphs. On the other hand, the affinity graph of LRSC is dense, so although each cluster is self-connected, there are abundant wrong connections. This highlights the advantage of our method, which is flexible in controlling the number of nonzero entries by adjusting the trade-off parameter $\lambda$. Our results illustrate that this trade-off improves clustering accuracy.

Finally, ENSC and KMP are two representatives of other methods that also exploit the trade-off between $\ell_1$ and $\ell_2$ regularizations. A drawback of both works is that the solvers for their optimization problems are not as effective as our ORGEN algorithm, as they cannot deal with large datasets due to memory requirements. Moreover, we observe that their algorithms converge to modest accuracy in a few iterations but can be very slow in giving a high precision solution. This may explain why their clustering accuracy is not as good as that of EnSC-ORGEN. Especially, we see that ENSC gives dense solutions although the true solution is expected to be sparser, and this is explained by the fact that the solution paths of their solver are dense solutions.

\section{Conclusion}
We investigated elastic net regularization (\ie, a mixture of the $\ell_1$ and $\ell_2$ norms) for scalable and provable subspace clustering. Specifically, we presented an active set algorithm that efficiently solves the elastic net regularization subproblem by capitalizing on the geometric structure of the elastic net solution.  We then gave theoretical justifications---based on a geometric interpretation for the trade-off between the subspace preserving and connectedness properties---for the correctness of subspace clustering via the elastic net. Extensive experiments verified that that our proposed active set method achieves state-of-the art clustering accuracy and can handle large-scale datasets.

\myparagraph{Acknowledgments}
C. You, D. P. Robinson and R. Vidal are supported by the National Science Foundation under grant 1447822. C.-G. Li is partially supported by National Natural Science Foundation of China under grants 61273217 and 61511130081, and the 111 project under grant B08004. The authors thank Ben Haeffele for insightful comments on the design of the ORGEN algorithm.$\!$


\numberwithin{equation}{section}
\numberwithin{figure}{section}
\numberwithin{table}{section}

\begin{appendices}
	\modify{
		The appendix is organized as follows. In Section \ref{sec:prf-elastic-net} we present proofs for the geometric properties of the elastic net solution. In Section \ref{sec:prf-alg} we show the convergence of algorithm ORGEN. In Section \ref{sec:prf-EnSC} we prove the relevant results for the properties of the EnSC. In Section \ref{sec:add-exp}, we use synthetically generated data to verify our results on the properties of the EnSC. In Section \ref{sec:lambda_eq_1}, we study the special case of $\lambda = 1$, in which the EnSC method reduces to SSC. We show that the properties of EnSC as well as the ORGEN algorithm also apply to SSC with minor modifications, thus this work also offers additional understanding of SSC. In Section \ref{sec:param} we report the parameters of the algorithms used for real data experiments. Finally, in Section \ref{sec:related-EnSC}, we clarify the contribution of this paper in comparison to several prior works on elastic net based subspace clustering.
	}
	\section{Proof of the Geometric Properties of the Elastic Net Solution}
	\label{sec:prf-elastic-net}
	
	A fundamental result that serves as the basis for the analysis of the elastic net solution in Section \ref{sec:elastic-net} is the next lemma. It is used to prove Theorem \ref{thm:geometry} and Propositions \ref{thm:Delta-dynamic-1} and \ref{thm:Delta-dynamic-2}.
	
	\begin{lemma}[\cite{Mol:JC09,Jin:IP09}] 		\label{thm:basic}
		The vector $\hat{\c} \in \RR^N$ is the unique solution to \eqref{eq:en} if and only if it satisfies
		\begin{equation}
			(1-\lambda) \hat{\c} = \T_\lambda\big(A^\transpose \cdot \gamma (\b - A \hat{\c})\big).
			\label{eq:basic}
		\end{equation}
	\end{lemma}
	
	\begin{proof}
		We provide a sketch of the proof for completeness. Since problem \eqref{eq:en} is strongly convex, $\hat{\c}$ is the unique optimal solution if and only if it satisfies the following optimality condition:
		\begin{equation}
			A^\transpose \cdot \gamma (\b - A \hat{\c}) = (1-\lambda) \hat{\c} + \lambda z.
			\label{eq:optimality}
		\end{equation}
		for some $z\in\partial \|\hat{\c}\|_1$.
		Then, by taking the soft-thresholding $\T_\lambda(\cdot)$ on both sides of \eqref{eq:optimality} we get \eqref{eq:basic}. For a proof of the reverse implication, suppose $\hat{\c}$ satisfies \eqref{eq:basic}. For each $j = 1, \cdots, N$, by considering the three cases $\hat{\c}_j > 0$, $\hat{\c}_j = 0$, and $\hat{\c}_j < 0$ separately, one can establish that the $j$-th row of \eqref{eq:optimality} is satisfied when the corresponding row of \eqref{eq:basic} holds.
	\end{proof}

	Theorem \ref{thm:geometry} follows trivially from this result. In the remainder of this section, we prove Propositions \ref{thm:Delta-dynamic-1} and \ref{thm:Delta-dynamic-2}.
	
	\subsection{Proof of Proposition~\ref{thm:Delta-dynamic-1}}	
	
	\begin{proof}
		Notice that $\c^*(\b, A)$ satisfies
		\begin{multline}
			(1-\lambda) \c^*(\b, A) = \T_\lambda\big(A^\transpose  \gamma (\b - A \c^*(\b, A))\big)\\
			= \T_\lambda\left(A^\transpose  \gamma \Big(\b - [A, A'] 
			\left[\begin{array}{c}
				\c^*(\b, A)\\
				\0_{N' \times 1}
			\end{array}\right]\Big)\right).
			\label{eq:prf-dyn1-1}
		\end{multline}
		Using the assumption that 
		no column of $A'$ is contained in $\Delta(\b, A)$, 
		it follows that
		\begin{equation*}
			\begin{split}
				(1-\lambda)\0_{N' \times 1} &= \T_\lambda\big(A'^\transpose  \bfdelta(\b, A)\big)\\
				&=\T_\lambda\big(A'^\transpose  \gamma (\b - A \c^*(\b, A))\big)\\
				&= \T_\lambda\left(A'^\transpose  \gamma \Big(\b - [A, A'] 
				\left[\begin{array}{c}
					\c^*(\b, A)\\
					\0_{N' \times 1}
				\end{array}\right]\Big)\right).
			\end{split}
		\end{equation*}
		We may then combine this equality with \eqref{eq:prf-dyn1-1} and define the vector $\hat{\c}:=[\c^*(\b, A)^\transpose, \0_{N' \times 1} ^\transpose] ^\transpose$ to obtain
		\begin{equation}
			(1-\lambda) \hat{\c} = \T_\lambda\Big([A, A']^\transpose  \gamma (\b - [A, A'] \hat{\c})\Big),
		\end{equation}
		thus by Lemma \ref{thm:basic}, $\hat{\c}$ must equal $\c^*(\b, [A, A'])$.
	\end{proof} 
	
	\subsection{Proof of Proposition~\ref{thm:Delta-dynamic-2}}		
	\begin{proof}
		We prove the contrapositive; let $\c_{A'} = \0$.  It then follows from $\c_A = \c^*(\b, A)$ that $\c^*(\b, [A, A']) = [\c^*(\b, A)^\transpose, \0^\transpose]$, and by definition of the oracle point that $\bfdelta(\b, A) = \bfdelta(\b, [A, A'])$. Now by Theorem \ref{thm:geometry}, we have
		\begin{equation}
			(1-\lambda)
			\left[\begin{array}{c}
				\c^*(\b, A)\\
				\0
			\end{array}\right]
			=\T_\lambda \left( 
			\left[\begin{array}{c}
				A^\transpose\\A'^\transpose
			\end{array}\right] \cdot \bfdelta(\b, A)\right).
		\end{equation}
		From the second block of equations and the definition of $\Delta(\b,A)$, we have that  no column of $A'$ lies in the oracle region $\Delta(\b, A)$, which completes the contrapositive proof.
	\end{proof}

	\section{Proof of Convergence for Algorithm~\ref{alg:main}}
	\label{sec:prf-alg}
	
	\subsection{Proof of Lemma~\ref{lem:decrease}}	
	
	\begin{proof}
		
		Let us define the sets
		\begin{align*}
			Q &:= T_k\setminus T_{k+1}, \\
			S &:= T_k\cap T_{k+1}, \ \ \text{and} \\ 
			R &:= T_{k+1} \setminus T_k \neq \emptyset,
		\end{align*}
		where the fact that $R$ is nonempty follows from the assumption $T_{k+1} \nsubseteq T_k$ in the statement of Lemma~\ref{lem:decrease}. By these definitions, $T_k = Q \cup S$, and $T_{k+1} = S \cup R$.
		
		By definition, $T_{k+1}$ contains all columns of $A$ that are in $\Delta(\b, A_{T_k})$, thus no column of $A_Q$ is in $\Delta(\b, A_{T_k})$. By Proposition \ref{thm:Delta-dynamic-1}, 
		\begin{equation}
			\c^*(\b, A_{T_k}) = \c^*(\b, [A_{S}, A_{Q}]) = 
			\left[\begin{array}{c}
				\c^*(\b, A_S)\\
				\0
			\end{array}\right],
			\label{eq:prf-decrease-step1}
		\end{equation}
		in which we have assumed without loss of generality that columns of $A_{T_k}$ are arranged in the order such that $A_{T_k} = [A_{S}, A_{Q}]$. Using \eqref{eq:prf-decrease-step1}, we have
		\begin{equation}
			\begin{split}
				&f(\c^*(\b, A_{T_k});~\b, A_{T_k})\\
				=&f\left( \left[\begin{array}{c}
					\c^*(\b, A_{S})\\
					\0
				\end{array}\right];~\b, [A_S, A_R] \right)\\
				\ge& \min_{\c} f(\c;~ \b, [A_S, A_R])\\
				=& f(\c^*(\b, [A_S, A_R]);~\b, [A_S, A_R])\\
				=& f(\c^*(\b, A_{T_{k+1}});~\b, A_{T_{k+1}}).
				\label{eq:prf-decrease-final}
			\end{split}
		\end{equation}
		It remains to show that the inequality in \eqref{eq:prf-decrease-final} is strict. We show this by arguing that $[\c^*(\b, A_S)^\transpose, \0^\transpose]^\transpose$ that appears on the second line of \eqref{eq:prf-decrease-final} is not an optimal solution to the optimization problem stated on the third line. Denote the solution to this optimization problem as
		\begin{equation}
			\c^*(\b, [A_S, A_R]) :=
			\left[\begin{array}{c}
				\c_S\\
				\c_R
			\end{array}\right],
		\end{equation}
		where $\c_S$ and $\c_R$ are of appropriate sizes. By \eqref{eq:prf-decrease-step1} and the definition of the oracle region, we have
		\begin{equation}
			\Delta(\b, A_S) = \Delta(\b, A_{T_k}).
		\end{equation}
		Combining this with the facts that the columns of $A_{T_{k+1}}$ are in $\Delta(\b, A_{T_k})$ and $R \subseteq T_{k+1}$, we know that the columns of $A_R$ are in $\Delta(\b, A_S)$. Consequently, by Proposition \ref{thm:Delta-dynamic-2}, we must have $\c_R \ne \0$. This shows that $[\c^*(\b, A_S)^\transpose, \0^\transpose]^\transpose$ is not an optimal solution to the problem on the third line of \eqref{eq:prf-decrease-final} and thus the inequality in \eqref{eq:prf-decrease-final} is strict.	
	\end{proof}

	\subsection{Proof of Theorem~\ref{thm:alg-convergence}}	
	\begin{proof}
		We first prove that Algorithm~\ref{alg:main} terminates in a finite number of iterations. We first 	observe that the objective is strictly decreasing during each iteration before termination occurs (see Lemma \ref{lem:decrease}). Since there are only finitely many different active sets, we must conclude that Algorithm~\ref{alg:main} terminates after a finite number of iterations with $T_{k+1} \subset T_k$.
		
		%
		
		We now prove that when Algorithm~\ref{alg:main} terminates, the output vector is optimal.  	Construct the vector $\hat{\c}$ such that $\hat{\c}_{T_k} = \c^*(\b, A_{T_k})$ and $\hat{\c}_{T_k^c} = 0$, in which $T_k^c$ is the complement of $T_k$ in $\{1, \cdots, N\}$. By Theorem \ref{thm:geometry}, for any $j \in T_{k}$ it holds that $(1-\lambda) \cdot c^*_j (\b, A_{T_k}) = \T_\lambda( \a_j ^\transpose \cdot \bfdelta(\b, A_{T_k}) )$. For any $j \notin T_{k}$, by the termination condition $T_{k+1} \subseteq T_k$ we know $j \notin T_{k+1}$. Thus, by step \ref{step:update-support}, $0 = \T_\lambda( \a_j ^\transpose \cdot \bfdelta(\b, A_{T_k}) )$. Consequently, $\hat{\c}$ satisfies the relation in \eqref{eq:basic} and thus is the solution, i.e., $\hat{\c} = \c^*(\b, A)$. Also, from the construction it can be seen that the support of $\hat{\c}$ is precisely $T_{k+1}$.
	\end{proof}

	\section{Proof of the Correctness of EnSC}
	\label{sec:prf-EnSC}
	In this section we prove the results in Section \ref{sec:EnSC}. 
	
	\subsection{Inradius}
	The inradius introduced in Definition \ref{def:inradius} characterizes the distribution of a set of points. The next lemma can be interpreted as giving an equivalent definition of inradius for certain convex sets. The result is used in interpreting differences between Theorem \ref{thm:subspace-preserving-condition-inradius} and Theorem \ref{thm:subspace-preserving-condition}, as well as in proving Theorem \ref{thm:bound-region}.
	
	\begin{lemma}\label{thm:inradius}
		If $\{\a _j\}_{j=1}^N$ are points with unit $\ell_2$ norm, then
		\begin{equation}
			r\big(\conv\{\pm\a _j\}_{j=1}^N\big) = \min_{\v \ne 0} \max_{j=1,\cdots,N}\mu(\a _j, \v).
			\label{eq:inradius-covering}
		\end{equation}
	\end{lemma}	
	\begin{proof}
		Let $A = [\a_1, \cdots, \a_N]$.  The right hand side of \eqref{eq:inradius-covering} can be written as
		\begin{multline}
			\min_{\v \ne 0} \max_{j=1,\cdots,N}\mu(\a _j, \v) \\
			= \min_{\v \neq 0} \frac{\|A^\transpose \v\|_\infty}{\|\v\|_2}
			= 1 / \max_{\v \neq 0} \frac{\|\v\|_2}{\|A^\transpose \v\|_\infty}.
			\label{eq:prf-inradius-covering}
		\end{multline}
		One then quotes the relation that the inradius of a symmetric convex body is the reciprocal of the circumradius of its polar set, which is exactly the right hand side of \eqref{eq:prf-inradius-covering} (see, e.g. Definition 7.2 in \cite{Soltanolkotabi:AS13} or Lemma 1 in \cite{You:arxiv15-SSR}).
	\end{proof}
	
	The interpretation of Lemma~\ref{thm:inradius} is as follows: one searches for a vector $\v$ that is furthest away from all points $\{\pm\a _j\}_{j=1}^N$, and the inradius is the coherence of this $\v$ with the closest neighbor in $\{\a _j\}_{j=1}^N$. In other words, it characterizes the covering property of the points $\{\pm\a _j\}_{j=1}^N$. If inradius is large then for any point in the space there exists an $\a _j$ that is close to it.
	
	\subsection{Proof of Lemma \ref{thm:subspace-preserving-lemma}}
	\begin{proof}
		Consider the problem
		\begin{equation}\label{eq:fictitious}
			\c^*(\x_j, X_{-j} ^\ell) = \arg\min_{\c} f(\c; \x_j, X_{-j} ^\ell)
		\end{equation}
		and by our notation, let $\Delta(\x_j, X_{-j} ^\ell)$ be its oracle region.
		
		For the ``if'' part, we know from Proposition \ref{thm:Delta-dynamic-1} that adding more points that are outside of the oracle region $\Delta(\x_j, X_{-j} ^\ell)$ to the dictionary of \eqref{eq:fictitious} does not affect its solution. To be more specific, if it holds that $\x_k \notin \Delta(\x_j, X_{-j} ^\ell)$ for all $\x_k \notin \S_\ell$, then by Proposition \ref{thm:Delta-dynamic-1} we have $\c^*(\b, X_{-j}) = P\cdot [\c^*(\x_j, X_{-j} ^\ell) ^\transpose, \0^\transpose] ^\transpose$, where $P$ is some permutation matrix.
		
		For the ``only if'' part, if any $\x_k \notin \S_\ell$ is in the oracle region $\Delta(\x_j, X_{-j} ^\ell)$, then Proposition \ref{thm:Delta-dynamic-2} shows that the coefficient vector of $\c^*(\b, X_{-j})$ that corresponds to points outside of $\S_\ell$  is nonzero.  Therefore, the solution is not correct in identifying the $l$-th subspace.
	\end{proof}
	
	\subsection{Proof of Theorem \ref{thm:bound-region}}
	
	Result Theorem \ref{thm:bound-region} follows from the bound on the norm of the oracle point given below in Lemma~\ref{thm:bound-delta} and the relation $\kappa\ge r$ as revealed by Lemma \ref{thm:inradius}.
	
	\begin{lemma}		\label{thm:bound-delta}
		Consider problem \eqref{eq:en}.  If we define $\kappa = \max_j \mu(\a_j, \bfdelta)$ as the coherence between the oracle point $\bfdelta$ and its closest neighbor among the columns of $A$, then
		\begin{equation}
			\|\bfdelta\|_2 \le \frac{\lambda \kappa + 1 - \lambda}{\kappa^2}.
			\label{eq:bound-delta}
		\end{equation}
	\end{lemma}
	
	\begin{proof}
		If $\c^* = 0$, then the optimality condition \eqref{eq:optimality} shows that $\|A^\transpose \bfdelta\|_\infty \le \lambda$, hence $\kappa  \|\bfdelta\|_2 \le \lambda$.  From this it is easy to see that \eqref{eq:bound-delta} holds.
		
		Next, we suppose that $\c^* \ne 0$, and assume without loss of generality that every entry in $\c^*$ is positive. (If an entry of $\c^*$ is zero then we can remove the corresponding column from $A$ without affecting the quantities $\bfdelta$ and $\kappa$.  Also, if $\c^*_j < 0$ for some $j$, we can change $\a_j$ to $-\a_j$ so that the solution will simply have $\c^*_j$ changed to $-\c^*_j$, which is then positive.) Since all entries of $\c^*$ are positive, we may conclude that $\a_j^\transpose \bfdelta > \lambda$ for all $j$.
		
		We now multiply both sides of the optimality condition \eqref{eq:optimality} by $\c ^{*\transpose}$ to obtain
		\begin{equation}
			\langle \c^*, A^\transpose \bfdelta \rangle = (1-\lambda) \|\c^*\|_2 ^2 + \lambda \|\c^*\|_1.
			\label{eq:prf-boundregion-1}
		\end{equation}
		Also, by the definition of the oracle point, we have
		\begin{equation}
			\langle A\c^*, \bfdelta \rangle = \langle \b - \bfdelta / \gamma, \bfdelta \rangle = \langle \b, \bfdelta \rangle - \|\bfdelta\|_2^2 / \gamma.
			\label{eq:prf-boundregion-2}
		\end{equation}
		Notice that since the left-hand-side of \eqref{eq:prf-boundregion-1} and \eqref{eq:prf-boundregion-2} are the same, we can equate the right-hand-sides to get
		\begin{align}
			(1- &\lambda) \|\c^*\|_2^2 + \lambda \|\c^*\|_1 \nonumber \\
			&= \langle \b, \bfdelta \rangle - \|\bfdelta\|_2^2 / \gamma \le \|\bfdelta\|_2 - \|\bfdelta\|_2^2/\gamma.
			\label{eq:prf-boundregion-3}
		\end{align}
		We now prove a lower bound on the left-hand-side of \eqref{eq:prf-boundregion-3} in terms of $\|\bfdelta\|_2$. From \eqref{eq:geometry} and $\a_j^\transpose \bfdelta > \lambda$ for all $j$, we have
		\begin{multline}
			(1-\lambda) \|\c^*\|_2 ^2 + \lambda \|\c^*\|_1 \ge (1-\lambda) c_j ^2 + \lambda c_j \\= \frac{\T_{\lambda} (\a_j ^\transpose \bfdelta) ^2}{1-\lambda} + \frac{\lambda \T_{\lambda} (\a_j ^\transpose \bfdelta)}{1-\lambda} = \frac{(\a_j ^\transpose \bfdelta - \lambda) \cdot \a_j ^\transpose \bfdelta}{1-\lambda}
		\end{multline}
		for all $1 \leq j \leq N$. If we now take $j$ to be the index that maximizes $\langle \a_j, \bfdelta / \|\bfdelta\|_2 \rangle$ and use the definition of $\kappa$, then 
		\begin{equation}
			(1-\lambda) \|\c^*\|_2 ^2 + \lambda \|\c^*\|_1 \ge \frac{(\kappa \|\bfdelta\|_2 - \lambda) \cdot \kappa \|\bfdelta\|_2}{1-\lambda}.
			\label{eq:prf-boundregion-4}
		\end{equation}
		Combining \eqref{eq:prf-boundregion-3} with \eqref{eq:prf-boundregion-4}, we get an inequality on $\|\bfdelta\|_2$:
		\begin{equation} 
			\frac{(\kappa \|\bfdelta\|_2 - \lambda) \cdot \kappa \|\bfdelta\|_2}{1-\lambda} \le \|\bfdelta\|_2 - \|\bfdelta\|_2^2/\gamma.
		\end{equation}
		This inequality gives a bound on $\|\bfdelta\|_2$ of
		\begin{equation}
			\|\bfdelta\|_2 \le \frac{\lambda \kappa + 1 - \lambda}{\kappa^2 + (1-\lambda)/\gamma} \le \frac{\lambda \kappa + 1 - \lambda}{\kappa^2},
		\end{equation}
		which completes the proof.
	\end{proof}
	
	\subsection{Proofs of Theorem \ref{thm:subspace-preserving-condition} and \ref{thm:subspace-preserving-condition-inradius}}
	
	Theorem \ref{thm:subspace-preserving-condition} can be obtained by combining Lemma \ref{thm:subspace-preserving-lemma} and Theorem \ref{thm:bound-region}. Theorem \ref{thm:subspace-preserving-condition-inradius} follows from Theorem \ref{thm:subspace-preserving-condition} and the fact that $\kappa_j \ge r_j$ as revealed by Lemma \ref{thm:inradius}.

	\section{Additional Experiments}
	\label{sec:add-exp}
	
	\subsection{Correctness of EnSC}
	
	In Theorem \ref{thm:subspace-preserving-condition-inradius} and Theorem \ref{thm:subspace-preserving-condition}, we give two conditions that guarantee the correctness of the representation given by EnSC for the purpose of subspace clustering. In this section, we use synthetic experiments to verify our theoretical analysis. Specifically, we verify that as the $\ell_1$-$\ell_2$ tradeoff parameter $\lambda$ increases, the representation is more likely to be correct. Moreover, we examine the tightness of our bound for predicting the correctness.
	
	For each pair of $N \in \{100, 200, 400, 800, 1600, 3200\}$ and $\lambda \in \{0.99, 0.95, 0.90, 0.80, 0.60, 0.40, 0.20, 0.10\}$, we randomly generate subspaces and data samples as specified in the caption of Figure \ref{fig:ssr}. We then run EnSC on the generated data matrix and get the representation vectors $\{ \c_j \}_{j=1} ^N$. In Figure \ref{fig:ssr_phase_experiment} we report the percentage of the $\c_j$ vectors that are correct in identifying its subspace. As can be seen, it is easier to get correct representations when $\lambda$ is larger. This is consistent with our intuition: as $\lambda$ becomes larger, the solution is sparser and is more likely to be correct. Moreover, this is consistent with what is predicted by our theoretical analysis, as in both Theorem \ref{thm:subspace-preserving-condition-inradius} and Theorem \ref{thm:subspace-preserving-condition} the condition for correctness is easier to be satisfied as $\lambda$ increases.
	
	We plot the result of Theorem \ref{thm:subspace-preserving-condition} in Figure \ref{fig:ssr_phase_theory}. Specifically, for each $j \in \{1, \cdots, N\}$, we can solve for $\c^*(\x_j, X_{-j} ^\ell)$ by using the ground truth labels and then compute $\bfdelta(\x_j, X_{-j} ^\ell)$ from $\c^*(\x_j, X_{-j} ^\ell)$ by \eqref{def:delta}. Consequently, all quantities in the condition of Theorem \ref{thm:subspace-preserving-condition} can be computed, and consequently whether or not the condition holds.   In Figure \ref{fig:ssr_phase_theory} we plot the percentage of points that satisfy the condition. Since our condition is sufficient but not necessary, we expect the percentage in Figure \ref{fig:ssr_phase_theory} to be no larger than the corresponding percentage in Figure \ref{fig:ssr_phase_experiment}, and the gap between them reveals the tightness of the result of Theorem \ref{thm:subspace-preserving-condition}. This gap is more clearly illustrated in Figure \ref{fig:ssr_curve}, in which we plot selected rows from Figure \ref{fig:ssr_phase_experiment} and \ref{fig:ssr_phase_theory} that correspond to $\lambda = \{0.99, 0.90, 0.60\}$. It can be seen that our condition becomes tighter as $\lambda$ approaches $1$.
	
	Finally, notice that while the condition in Theorem \ref{thm:subspace-preserving-condition} can be checked when the ground truth is known, the condition in Theorem \ref{thm:subspace-preserving-condition-inradius} cannot be since it is generally NP-hard to compute the inradius $r_j$ \cite{Soltanolkotabi:AS13}. This is an advantage of Theorem \ref{thm:subspace-preserving-condition}, in addition to the fact that it has a weaker requirement to guarantee the correctness of EnSC. 
	
	\begin{figure*}[t]
		\centering
		\subfigure[\label{fig:ssr_phase_experiment}Percentage of correctness by experiment.]{\includegraphics[scale = 0.44]{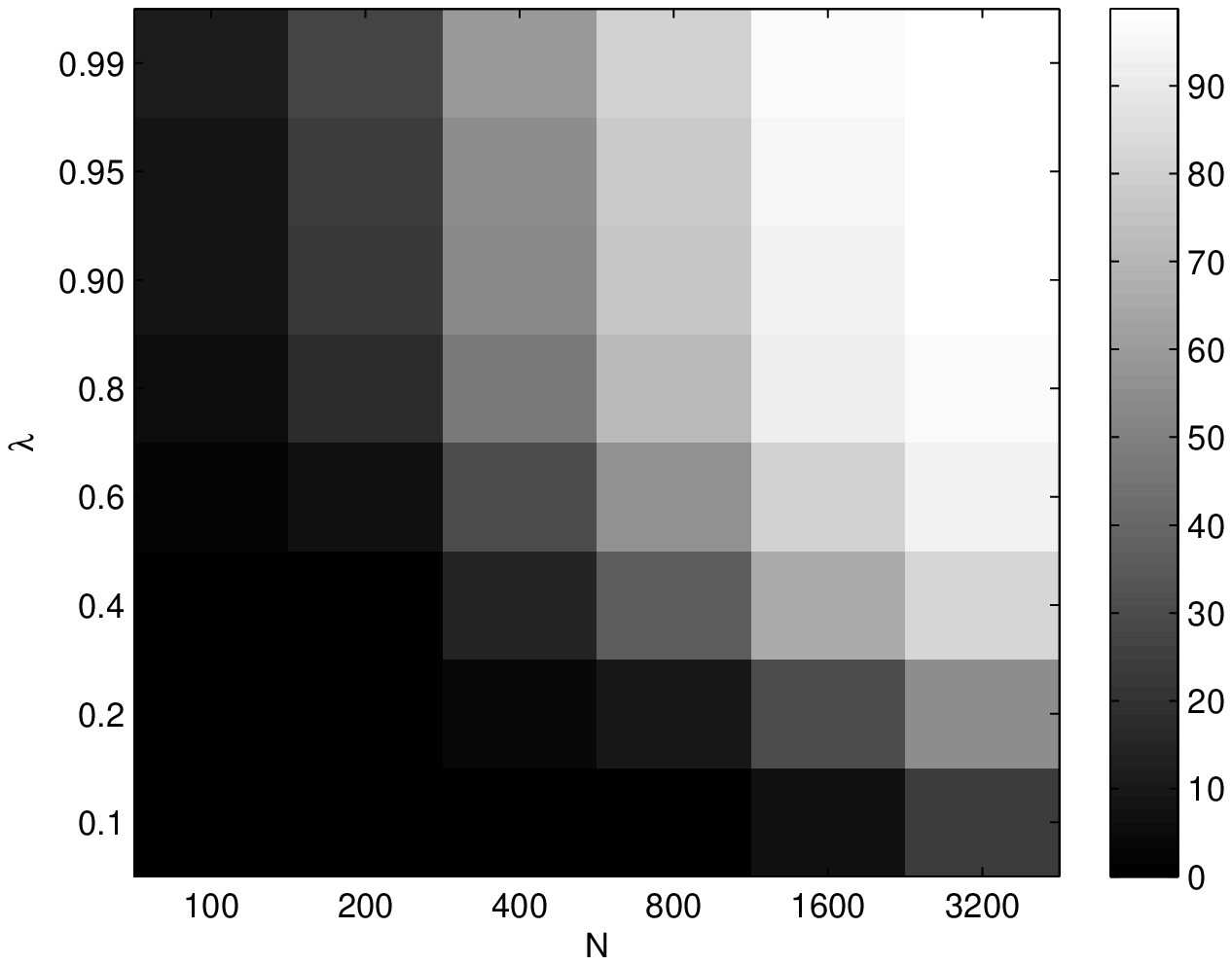}}
		~
		\subfigure[\label{fig:ssr_phase_theory}Percentage of correctness by analysis.]{\includegraphics[scale = 0.44]{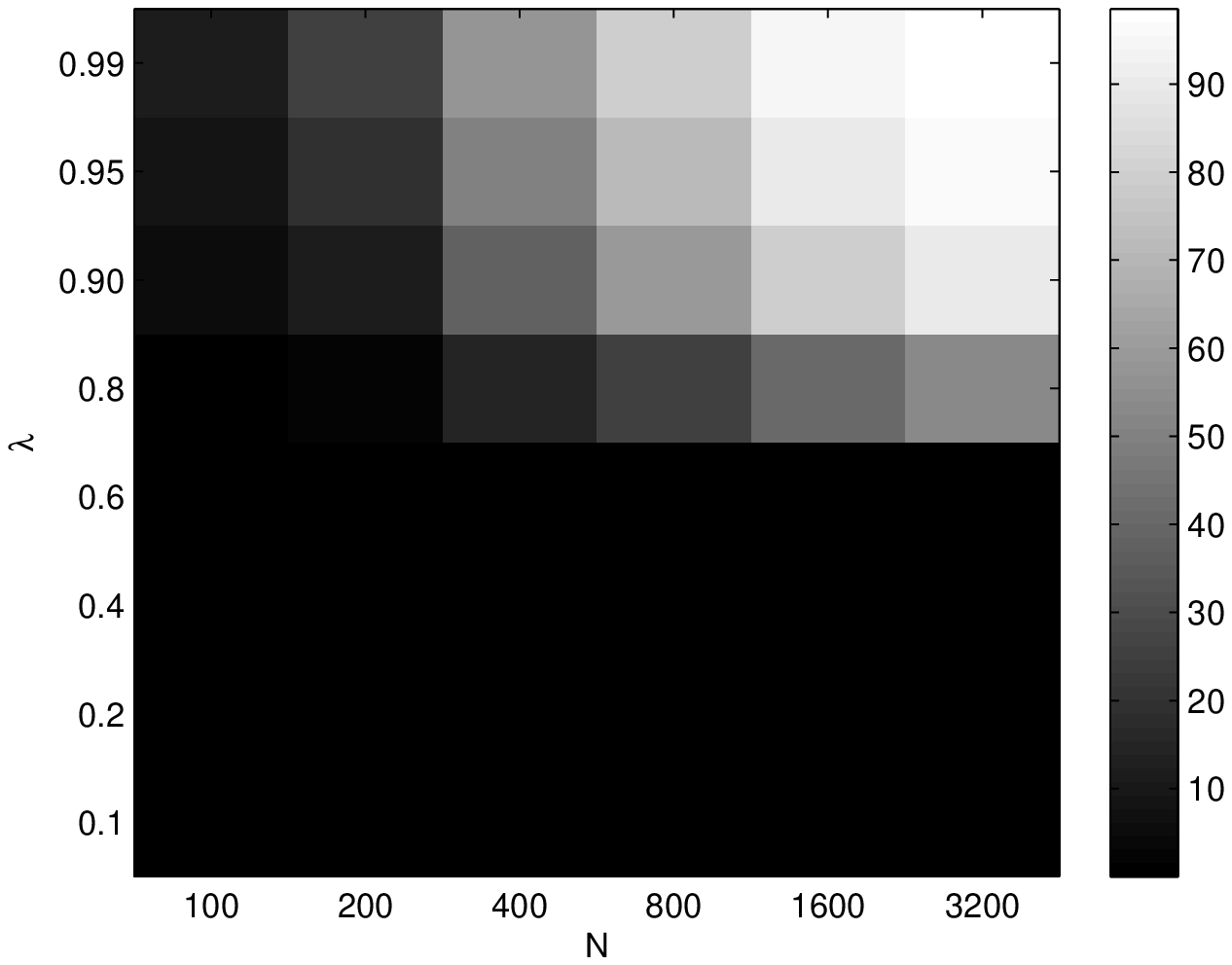}}
		~
		\subfigure[\label{fig:ssr_curve}]{\includegraphics[scale = 0.50]{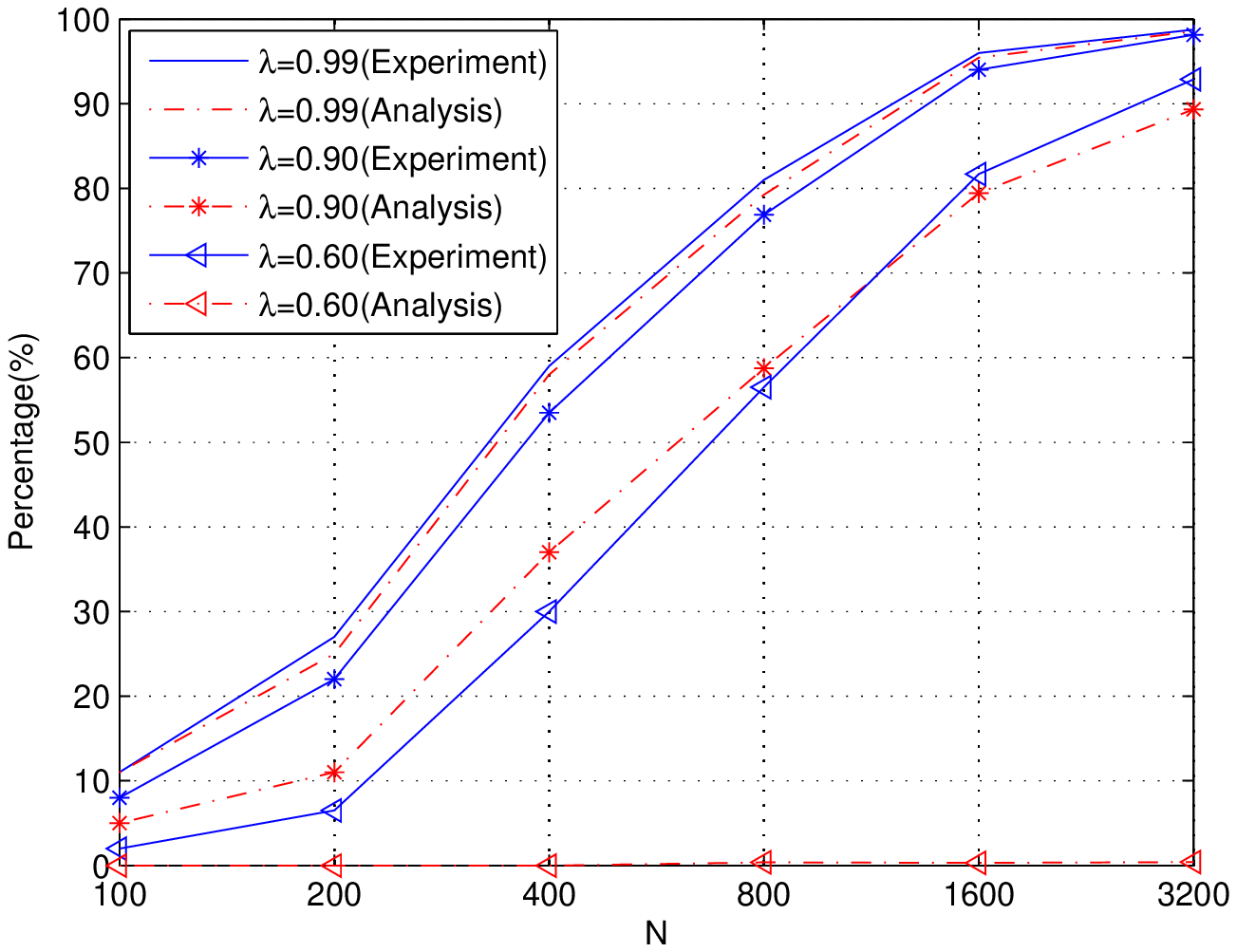}}
		\caption{Correctness of the solution of EnSC for different values of $\lambda$. We generate $4$ subspaces each of dimension $8$ in an ambient space of dimension $20$ uniformly at random. On each subspace, we sample uniformly at random an equal number of points that add up to $N$, which varies from $100$ to $3200$. We report the percentage of representations that are correct in identifying subspaces. (a, b) The percentage of correct representations for different values of $\lambda$ and $N$ as produced by experimental results and as predicted by Theorem \ref{thm:subspace-preserving-condition}, respectively. (c) Plots of selected rows from (a) and (b) to help clarify the difference.} 
		\label{fig:ssr}
	\end{figure*}

	\section{Discussion for the Case $\lambda=1$}
	\label{sec:lambda_eq_1}
	As the analyses and results of this paper are for $\lambda\in [0, 1)$, in this section we discuss the case $\lambda = 1$. It turns out that the geometric structure of the elastic net solution for $\lambda = 1$ is slightly different. As a result, many of the theorems and discussions do not apply for $\lambda = 1$, so that we need a separate discussion for most of the results.
	
	\myparagraph{The oracle point and oracle region} We use the same definitions of the oracle point and oracle region as before. While for $\lambda \in [0, 1)$ the oracle point $\bfdelta$ is unique since $\c^*$ is unique due to the strong convexity of the problem, the same argument does not apply to the case $\lambda = 1$. However, we can sill establish the uniqueness of the oracle point.  
	
	\begin{theorem}
		The oracle point $\bfdelta(\b, A)$ is unique for each choice of  $\lambda \in  [0, 1]$.
	\end{theorem}
	\begin{proof}
		For $\lambda < 1$, the optimization problem~\eqref{eq:en} is strongly convex, thus $\c^*$ is unique. Then, by \eqref{eq:def-delta}, $\bfdelta(\b,A)$ is unique.
		
		For $\lambda = 1$, we rewrite problem \eqref{eq:en} equivalently as
		
		\begin{equation}
			\min_{\c, \e} \|\c\|_1 + \frac{\gamma}{2} \|\e\|_2^2 \st \b = A \c + \e.
		\end{equation}
		Introducing the dual vector $\v$, the Lagragian function is
		\begin{equation}
			L(\c, \e, \v) = \|\c\|_1 + \frac{\gamma}{2} \|\e\|_2^2 + \langle \v, \b - A \c - \e \rangle,
		\end{equation}
		and the corresponding dual problem is
		\begin{equation}
			\max_{\v} \, \langle \b, \v \rangle - \frac{1}{2\gamma}\v^\transpose \v \st \|A^\transpose \v\|_\infty \le 1,
			\label{eq:en-dual}
		\end{equation}
		whose objective function is strongly concave with a unique solution $\v^*$. Also, from the optimality conditions we have $\v^* = \gamma \e^* = \gamma (\b - A \c^*(\b,A)) = \bfdelta(\b,A)$, so that $\bfdelta(\b,A)$ is unique. 
	\end{proof}
	
	\myparagraph{The geometric structure of the solution} Recall that from Theorem \ref{thm:geometry} we know that the oracle region contains points whose corresponding coefficients are nonzero, \ie, $\c^*_j \ne 0$ if and only if $\a_j \in \Delta(\b, A)$. For the case $\lambda=1$, this argument no longer holds. Actually, Theorem \ref{thm:geometry} still holds for $\lambda = 1$, but the left-hand-side of \eqref{eq:basic} becomes zero, and it means that no column of $A$ is in the oracle region $\Delta(\b, A)$. To further understand the structure of the solution, we need the following result.
	
	\begin{theorem}
		The solution $\c^* = \c^*(\b,A)$ to problem \eqref{eq:en} with $\lambda = 1$ satisfies that
		if $\c^*_j \ne 0$, then $|\a_j^\transpose \bfdelta| = 1$.
	\end{theorem}
	
	This result follows from the optimality condition. It means that a coefficient $\c^*_j$ is nonzero only if $\a_j$ is on the boundary of the oracle region $\Delta(\b, A)$, which we denote as $\partial \Delta(\b, A)$. The opposite is generally not true: if $\a_j \in\partial \Delta(\b, A)$, it does not necessarily mean that $\c^*_j \neq 0$.
	
	The geometric structure of the solution is thus clear: all columns of $A$ are outside the oracle region, but some columns of $A$ are in $\partial \Delta(\b,A)$ with some of these corresponding to nonzero coefficients. 
	
	\myparagraph{The ORGEN algorithm} Algorithm \ref{alg:main} needs to be revised when $\lambda = 1$. Specifically, we need an alternative step \ref{step:update-support}:
	\begin{equation} \label{eq:alg-step5-general}
		\text{\ref{step:update-support}''}\!: T_{k+1} \leftarrow \{ j: \a_j \in \Delta(\b, A_{T_k})\} \cup S_k, 
	\end{equation}
	where $S_k = \{ j: [\c^*(\b, A_{T_k})]_j \ne 0 \}$ is the support of $\c^*(\b, A_{T_k})$. Notice that $S_k\subseteq \partial\Delta(\b,A)$ when $\lambda = 1$ so that the two operands in the union in \eqref{eq:alg-step5-general} are disjoint sets. With this modification, one can show that ORGEN converges to an optimal solution in a finite number of iterations. The proof is essentially the same as before and omitted here. In the case when the solution is not unique, the solution that ORGEN converges to depends upon the initialization $T_0$ as well as the specific solution given by the solver in step \ref{step:solve-subproblem}.
	
	For $\lambda \in [0, 1)$, $S_k \subseteq \Delta(\b, A_{T_k})$ by the definition of the oracle region. Thus, the alternative step specified by \eqref{eq:alg-step5-general} applies to any $\lambda \in [0, 1]$. We write this as a theorem.
	
	\begin{theorem}
		Algorithm \ref{alg:main} with the alternative step \ref{step:update-support} specified in \eqref{eq:alg-step5-general} converges to an optimal solution $\c^*(\b, A)$ in a finite number of iterations for all $\lambda \in [0, 1]$.
	\end{theorem}
	
	\myparagraph{Correctness of EnSC} Theorem \ref{thm:subspace-preserving-condition} gives a sufficient condition for guaranteeing the correctness of EnSC when $\lambda \in [0, 1)$. In extending the result to the case $\lambda = 1$ we need a slightly stronger condition.
	
	\begin{theorem}\label{thm:subspace-preserving-condition-general}
		Let $\x_j \in \S_\ell$, and $\bfdelta_j$ and $\kappa_j$ be defined as in Theorem \ref{thm:subspace-preserving-condition}. Then, for all $\lambda \in [0, 1]$, the solution $\c^*(\x_j, X_{-j})$ is correct in identifying the subspace $\S_\ell$ if 	
		\begin{equation}	\label{eq:subspace-preserving-condition-general}
			\max_{k: \x_k \notin \S_\ell}\mu(\x_k, \bfdelta_j) < \frac{\kappa_j^2}{\kappa_j + \frac{1-\lambda}{\lambda}}.
		\end{equation}
	\end{theorem}
	
	The difference between \eqref{eq:subspace-preserving-condition-general} and \eqref{eq:subspace-preserving-condition} is that the inequality is strict in \eqref{eq:subspace-preserving-condition-general}. This modification is necessary to handle the case $\lambda = 1$, for the condition \eqref{eq:subspace-preserving-condition} does not exclude the case that $\x_k \notin \S_\ell$ may lie on the boundary of $\Delta(\x_j, X_{-j}^\ell)$ and yet correspond to a nonzero coefficient. 
	
	Finally, we discuss the implication of Theorem \ref{thm:subspace-preserving-condition-general} in the context of SSC. When $\lambda = 1$, condition \eqref{eq:subspace-preserving-condition-general} simplifies to
	\begin{equation}
		\max_{k: \x_k \notin \S_\ell}\mu(\x_k, \bfdelta_j) < \kappa_j.
		\label{eq:subspace-preserving-condition-SSC}
	\end{equation}
	In \cite{Soltanolkotabi:AS13} a sufficient condition for SSC is given by
	\begin{equation}
		\max_{k: \x_k \notin \S_\ell}\mu(\x_k, \bfdelta_j) < r_j.
		\label{eq:subspace-preserving-condition-SSC-prior}
	\end{equation}
	Using the relationship $r_j \le \kappa_j$, our condition in \eqref{eq:subspace-preserving-condition-SSC} is a weaker requirement than that in the previous work. Specifically, condition \eqref{eq:subspace-preserving-condition-SSC-prior} requires that the entire subspace $\S_\ell$ is well-covered by the columns of $X_{-j}^\ell$ so that $r_j$ is large. In contrast, our condition in \eqref{eq:subspace-preserving-condition-SSC} only requires the neighborhood of the oracle point $\bfdelta_j$ to be well-covered, \ie, that there exists a column in $X_{-j}^\ell$ that is close to $\bfdelta_j$. Another advantage of our condition \eqref{eq:subspace-preserving-condition-SSC} is that it can be verified when the ground truth is known. In contrast, the condition in \eqref{eq:subspace-preserving-condition-SSC-prior} cannot be verified since the computation of $r_j$ is generally NP-hard \cite{Soltanolkotabi:AS13}.
	
	\section{Parameters for Experiments on Real Data}
	\label{sec:param}
	
	\begin{table*}[t]
		\centering
		\caption{Parameters for experiments on real data.}
		\label{tbl:parameters}
		\begin{tabular}{c|c|cc|ccc|cc|cc}
			\hline
			&   SSC    & \multicolumn{2}{|c|}{LRSC} & \multicolumn{3}{|c|}{ENSC}              & \multicolumn{2}{|c|}{KMP} & \multicolumn{2}{|c}{EnSC-ORGEN} \\ \cline{2-11}
			& $\alpha$ & $\tau$ &     $\alpha$      & $\lambda_1$ & $\lambda_2$ & $\lambda_3$ & $k$ &      $\lambda$      & $\lambda$ &       $\alpha$       \\ \hline\hline
			Coil-100 &    25    &   5    &         5         &     0.1     &     0.1     &      1      & 100 &         0.1         &   0.95    &          3           \\ \hline
			PIE    &   200    &  100   &        100        &     0.1     &     0.1     &    1000     & 100 &         0.1         &    0.1    &         200          \\ \hline
			MNIST   &   120    &   -    &         -         &      -      &      -      &      -      &  -  &          -          &   0.95    &         120          \\ \hline
			CovType  &    -     &   -    &         -         &      -      &      -      &      -      &  -  &          -          &   0.95    &          50          \\ \hline
		\end{tabular}
	\end{table*}
	For the purpose of reproducible results, we report the parameters used for all the methods in the real data experiments. TSC, OMP and NSN all have a parameter that controls the number of nonzero coefficients in the representation. This parameter is the same as the  ``sparsity'' reported in Table \ref{tbl:subspace-real}. The NSN has two additional parameters. One is the maximum subspace dimension, for which we set as the default value suggested by the original paper. The other is $\epsilon$ which controls a post-processing step. For the purpose of a fair comparison with other methods, we set $\epsilon=0$ which essentially disables this post-processing step. The SSC-SPAMS uses the model in \eqref{eq:self-expression} with $r(\cdot) = \|\cdot\|_1$, $h(\cdot) = \frac12\|\cdot\|_2^2$, and $\gamma = \alpha \cdot \gamma_0$, where $\alpha$ is a hyperparameter specified in Table \ref{tbl:parameters} and $\gamma_0$ is the smallest value of $\gamma$ such that $\c^*(\x_j, X_{-j})$ is nonzero. The parameters for the solver SPAMS are set to their default values. 
	
	For SSC-ADMM we use the code for solving the optimization problem (13) as presented in \cite{Elhamifar:TPAMI13}, with $\alpha$ set to the same value as for SSC-SPAMS. For LRSC we use the code for model (P3) in \cite{Vidal:PRL14}, in which the parameters $\tau$ and $\alpha$ are provided in Table \ref{tbl:parameters}. ENSC has the three key parameters $\lambda_1, \lambda_2$, and $\lambda_3$ in their model (see \cite[equation (4)]{Panagakis:PRL14}). The remaining  parameters were set as suggested by the authors.  For KMP we implemented \cite[Algorithm 1]{Lai:ECCV14} in which the number of iterations $T$ is set to be 150, the parameter $L$ is set to be $1.1$ times the Lipschitz constant, and $k$ and $\lambda$ are reported in Table \ref{tbl:parameters}.
	
	Finally, for the proposed algorithm EnSC-ORGEN, the parameter $\lambda$ controls the trade-off between the $\ell_1$ and $\ell_2$ norms, and the parameter $\alpha$ controls the value for $\gamma$ in \eqref{eq:def-f} via the definition $\gamma = \alpha \gamma_0$, where $\gamma_0$ is the smallest value such that $\c^*(\x_j, X_{-j})$ is nonzero. The parameters are summarized in Table \ref{tbl:parameters}.

	\modify{
		\section{Relation with prior work on EnSC}
	}
	\label{sec:related-EnSC}
	The elastic net formulation was originally proposed in \cite{Zou:JRSS05} and subsequently introduced to subspace clustering in \cite{Fang:ICDM12,Fang:TKDE15,Panagakis:PRL14}. In these works, the regularization $r(\cdot)$ in \eqref{eq:self-expression} is set to be the $\ell_1$-$\ell_2$ combination as in \eqref{eq:r-l1_l2}. For the penalty function $h(\cdot)$, \cite{Panagakis:PRL14} proposes to use the $\ell_1$ penalty, while \cite{Fang:ICDM12} uses a joint $\ell_1$-$\ell_2$ penalty. Both works use existing methods for solving their optimization problem: \cite{Fang:ICDM12} uses the accelerated proximal gradient (APG) \cite{Beck2009} and \cite{Panagakis:PRL14} uses the linearized alternating direction method (LADM) \cite{Lin:NIPS11}.
	
	The optimization model studied here is slightly different from these prior works since we set $h(\e)$ to be the $\ell_2$ penalty as suggested in the original elastic net paper. Despite this difference in modeling the noise, all three models use the elastic net regularization. The major contributions of our work in comparison to these related works are threefold:
	
	\begin{enumerate}
		\item We design a new active-set algorithm for solving the optimization problem. In comparison to APG and LADM that are used in the related works, our method is computationally more efficient, and is able to handle larger datasets.
		
		\item Although using the elastic net for subspace clustering to balance correctness and connectivity is not new, we provide the first detailed argument based on a geometric interpretation of the solution of the elastic net. This deepens the understanding of the approach.
		
		\item We provide (under general conditions) the first proof of correctness for elastic net based subspace clustering.
	\end{enumerate}

\end{appendices}

{\small
\bibliographystyle{ieee}
\bibliography{biblio/vidal,biblio/vision,biblio/math,biblio/learning,biblio/sparse,biblio/geometry,biblio/dti,biblio/recognition,biblio/surgery,biblio/coding,biblio/segmentation}
}

\end{document}